\documentclass[letterpaper]{article}

\usepackage[round]{natbib}
\setcitestyle{authoryear,open={(},close={)}} 

\usepackage[english]{babel}
\usepackage{times}
\usepackage{helvet}
\usepackage{courier}
\usepackage[hyphens]{url}
\usepackage{graphicx}
\usepackage{graphicx}
\usepackage[dvipsnames]{xcolor}
\usepackage{xspace,amstext,amssymb,url}
\usepackage{subcaption}
\usepackage{tikz}
\usetikzlibrary{calc}
\usetikzlibrary{shapes.geometric}
\usepackage{amsmath}
\usepackage{booktabs}
\usepackage[dvipsnames]{xcolor}

\usepackage{apxproof}

\usepackage{enumerate}

\newcommand{\nc}{\newcommand}
\nc{\cF}{{\mathcal F}}
\nc{\cM}{{\mathcal M}}
\nc{\cO}{{\mathcal O}}
\nc{\cP}{{\mathcal P}} 
\nc{\cR}{{\mathcal R}} 
\nc{\nat}{\mathbb{N}}
\nc{\+}[1]{\mathcal{#1}} 

\nc{\lab}{\textit{lab}}
\nc{\true}{``true''\xspace}
\nc{\false}{``false''\xspace}

\newcommand{\tileN}{\lozenge}
\newcommand{\tileY}{\blacklozenge}
\newcommand{\tilen}{\mbox{$\vartriangle$}}\newcommand{\tiley}{\blacktriangle}
\nc{\querycon}{\text{Containment}\xspace}
\nc{\univ}{\textsf{RegUniversality}\xspace}

\newcommand{\set}[1]{\{#1\}}

\makeatletter
\newcommand\xleftrightarrow[2][]{  \ext@arrow 9999{\longleftrightarrowfill@}{#1}{#2}}
\newcommand\longleftrightarrowfill@{  \arrowfill@\leftarrow\relbar\rightarrow}
\makeatother

\newcommand{\yit}{\ensuremath{y_{i,t}}\xspace}
\newcommand{\yif}{\ensuremath{y_{i,f}}\xspace}
\newcommand{\yitf}{\ensuremath{y_{i,tf}}\xspace}
\newcommand{\yktf}{\ensuremath{y_{k,tf}}\xspace}
\newcommand{\crpq}{\textup{CRPQ}\xspace}
\newcommand{\crpqs}{\textup{CRPQs}\xspace}

\newcommand{\ans}{ans\xspace}
\newcommand{\defstyle}[1]{\textbf{#1}}

\renewcommand{\epsilon}{\varepsilon}

\newcommand{\ptime}{{\sc Ptime}\xspace}
\nc{\pitwo}{\ensuremath{\Pi^p_2}\xspace}
\nc{\sigmatwo}{\ensuremath{\Sigma^p_2}\xspace}
\newcommand{\np}{{\sc NP}\xspace}
\newcommand{\conp}{{\sc coNP}\xspace}
\newcommand{\pspace}{{\sc PSpace}\xspace}

\newcommand{\expspace}{{\sc Exp\-Space}\xspace}

\newtheoremrep{theorem}{Theorem}[section]

\newtheoremrep{lemma}[theorem]{Lemma}
\newtheoremrep{proposition}[theorem]{Proposition}
\newtheoremrep{claim}[theorem]{Claim}
\newtheoremrep{observation}[theorem]{Observation}
\newtheorem{remark}[theorem]{Remark}

\newtheorem{corollary}[theorem]{Corollary}

\usepackage[backgroundcolor=orange!50, textsize=small, disable]{todonotes} 
\newcommand{\wimColor}{BrickRed}
\newcommand{\tinaColor}{blue}

\newcommand{\matthiasColor}{ForestGreen}

\newcommand{\tina}[1]{\todo[inline,color=\tinaColor!20]{{\bf Tina:} #1}}

\newcommand{\wim}[1]{\todo[inline,color=\wimColor!20]{{\bf Wim:} #1}}

\newcommand{\matthias}[1]{\todo[inline,color=\matthiasColor!20]{{\bf Matthias:} #1}}

\newlength\boxwidth
\newlength\questionwidth
\newcommand{\decisionproblem}[4]{
    \setlength\boxwidth{#1}{
        \setlength\questionwidth{#1}\addtolength\questionwidth{-2.5cm}{
            \begin{center}
                \fbox{\parbox[t]{\boxwidth}{\centerline{#2}
                        \vspace{1ex}
                        \begin{tabular}{r@{\hspace{5mm}}p{\questionwidth}} 
                            Given: & #3\\[1pt]
                            Question: & #4
                \end{tabular}}}
            \end{center}
}}}

\begin{document}
\title{Containment of\\ Simple Conjunctive Regular Path Queries}

\author{
  Diego Figueira\thanks{Univ.\ Bordeaux, CNRS, Bordeaux INP, LaBRI, UMR 5800, Talence, France}
  \and
  Adwait Godbole\thanks{IIT Bombay, Mumbai, India}
  \and
  S. Krishna\footnotemark[2]
  \and
  Wim Martens\thanks{University of Bayreuth, Bayreuth, Germany}
  \and
  Matthias Niewerth\footnotemark[3]
  \and
  Tina Trautner\footnotemark[3]
}

\maketitle
\begin{abstract}
  Testing containment of queries is a fundamental reasoning task in knowledge
  representation.         We study here the containment problem for
  Conjunctive Regular Path Queries (CRPQs), a navigational query language
  extensively used in ontology and graph database querying. While it is known
  that containment of CRPQs is \expspace-complete in general, we focus here on
  severely restricted fragments, which are known to be highly relevant in
  practice according to several recent studies. We obtain a detailed overview of
  the complexity of the containment problem, depending on the
  features used in the regular expressions of the queries, with completeness
  results for
   \np, \pitwo, \pspace or \expspace.
\end{abstract}

\makeatletter{}

\section{Introduction}

Querying knowledge bases is one of the most important and fundamental tasks in
knowledge representation. Although much of the work on querying knowledge bases
is focused on conjunctive queries, there is often the need to use a simple form
of recursion, such as the one provided by regular path queries (RPQ), which ask
for paths defined by a given regular language. Conjunctive RPQs (CRPQs) can then
be understood as the generalization of conjunctive queries with this form of
recursion. CRPQs are part of SPARQL, the W3C standard for querying RDF data,
including well known knowledge bases such as DBpedia and Wikidata. In
particular, RPQs are quite popular for querying Wikidata. They are used in over 24\% of
the queries (and over 38\% of the unique queries), according to recent
studies~\citep{MalyshevKGGB18,BonifatiMT-www19}.
More generally, CRPQs are basic building blocks for querying
graph-structured databases \citep{Barcelo-pods13}.

As knowledge bases become larger, reasoning about queries (e.g.\ for
optimization) becomes increasingly important. 
One of the most basic reasoning tasks is that of query containment: is every
result of query $Q_1$ also returned by $Q_2$? This can be a means
for query optimization, as it may avoid evaluating parts of a query, or reduce
and simplify the query with an equivalent one. Furthermore, query containment
has proven useful in knowledge base verification, information integration,
integrity checking, and cooperative answering \citep{CalvaneseGLV-kr00}.

The containment problem for CRPQ is \expspace-complete, as was shown by
\citep{CalvaneseGLV-kr00} in a now `classical' KR paper, which appeared 20 years
ago. However, the lower bound construction of \citeauthor{CalvaneseGLV-kr00}
makes use of CRPQs which have a simple shape (if seen as a graph of atoms) but
contain rather involved regular expressions, which do not correspond to RPQs how
they typically occur in practice. Indeed, the analyses of
\citep{BonifatiMT-www19,BonifatiMT-vdlbj20} reveal that a large majority of
regular expressions of queries used in practice are of a very simple form.
This motivates us to revisit CRPQ containment on queries, focusing on commonly
used kinds of regular expressions. Our goal is to identify restricted fragments
of \crpqs that are common in practice and which have a reasonable complexity for
query containment.

\smallskip \noindent \textbf{Contribution.}
According to recent studies on query logs,
investigating over 500 million SPARQL queries
\citep{BonifatiMT-www19,BonifatiMT-vdlbj20}, it turns out that a large majority
of regular expressions that are used for graph navigation are of rather simple
forms, like $a^*$,
$ab^*$, $(a+b)c^*$, $a(b+c)^*d$, i.e., concatenations of (disjunctions of)
single symbols and Kleene stars of (disjunctions of) single symbols. Since CRPQs
have concatenations built-in, CRPQs with such expressions are essentially CRPQs
in which every atom has a regular expression of the form $(a_1 + \cdots + a_n)$ or $(a_1 + \cdots + a_n)^*$ for $n \geq 1$. In the remainder of the paper, we often abbreviate the
former type of atom with $A$ and the latter by $A^*$. If $n=1$, we write $a$ and
$a^*$. Table~\ref{tab:percentages} gives an overview of the frequency of such
expressions in the following data sets:
\begin{enumerate}[(a)]
\item The data set studied by \citep{BielefeldtGK18,BonifatiMT-www19}, which was
  released by \citep{MalyshevKGGB18} and contains 208 million parseable Wikidata
  queries, with over 55 million regular path queries.
\item The data set of \citep{BonifatiMT-vdlbj20}, which contains 339 million
  parseable queries, mostly from DBpedia, but also from LinkedGeoData,
  BioPortal, OpenBioMed, Semantic Web Dog Food and the British Museum. These
  queries contain around 1.5 million regular path queries.\footnote{One sees
    that regular path queries are much more common in the Wikidata log than in
    the (mainly) DBpedia log. The reason for this is that the graph structure of DBpedia
    was designed before RPQs (property paths) existed in SPARQL.}
\end{enumerate}
When we list multiple types of atoms in the table, we allow concatenations of
these types. So, $a(b+c)^*d$ is of type $a,A^*$ and also of the more general
type $A,A^*$.

Another motivation to study CRPQs with atoms of the forms $a, a^*, A$, and $A^*$
is that these are currently the only expressible atoms in CRPQs in Cypher~9~\cite[Figure
3]{FrancisGGLLMPRS-sigmod18}, a popular query language for property graphs.

We study the complexity of CRPQ containment for such fragments $\cF$ of ``simple
CRPQs'', that is, CRPQs that only use atoms of some of the types $a, a^*, A$,
and $A^*$. For each fragment $\cF$, we provide a complete picture of the
complexities of containment problems of the form $\cF \subseteq \cF$, $\cF
\subseteq \crpq$, and $\crpq \subseteq \cF$ (cf.\
Table~\ref{tab:complexity-smaller}, which we discuss in
Section~\ref{sec:mainresults} in detail). The main take-aways
are: \begin{enumerate}
\item Even for such simple CRPQs, containment of the form $\cF \subseteq \cF$
  can become \expspace-complete. Moreover, this lower bound already holds for
  containment of CRPQs using only $a$-atoms and $A^*$-atoms. This was surprising
  to us, because such CRPQs seem at first sight to be only mild extensions of
  conjunctive queries: they extend conjunctive queries only with atoms of the
  form $(a_1+\cdots+a_n)^*$, i.e., Kleene closures over sets of symbols. The
  contrast between \np-completeness of containment for conjunctive queries and
  \expspace-completeness for CRPQs that additionally allow $(a_1+\cdots+a_n)^*$
  is quite striking.
\item As soon as we disallow disjunction within Kleene closures in $\cF$, the
  complexity of the abovementioned containment problems drops drastically to $\pitwo$ or \pspace.
  \matthias{Why \pspace? This is only with general CRPQs on the left.} \wim{I
    meant the complexity of the three problems above the itemize. I added some
    more info here. Is it clearer? I understand that it's a bit amiguous.}
  The good news is that such regular expressions are still extremely common in
  practice, e.g., over 98\% of the RPQs in the Wikidata query logs
  (Table~\ref{tab:percentages}).
\end{enumerate}

\begin{table*}
  \centering
  \begin{tabular}{@{}crrrr@{}}
    \toprule
    & \multicolumn{4}{c}{Wikidata Queries} \\
    & \multicolumn{2}{c}{One-way RPQs} & \multicolumn{2}{c}{Two-way RPQs}\\
    RPQ Class & Valid \% & Unique \% & Valid \% & Unique \%\\ 
    \midrule
    $A,A^*$ & 99.02\% & 98.73\% & 99.83\% & 99.83\%\\
    $A,a^*$ & 98.40\% & 98.31\% & 99.22\% & 99.44\%\\
    $a,A^*$ & 93.50\% & 95.99\% & 94.30\% & 97.10\%\\
    $a,a^*$ & 92.88\% & 95.58\% & 93.69\% & 96.69\%\\
    \midrule
    Total & 55,333K & 14,189K & 55,333K & 14,189K\\ 
    \bottomrule
  \end{tabular}
  \quad
  \begin{tabular}{@{}rrrr@{}}
    \toprule
    \multicolumn{4}{c}{DBpedia$^\pm$ Queries}\\
    \multicolumn{2}{c}{One-way RPQs} & \multicolumn{2}{c}{Two-way RPQs}\\
    Valid \% & Unique \% & Valid \% & Unique \%\\
    \midrule
    68.99\% & 47.41\% & 94.35\% & 82.86\%\\
    65.29\% & 46.02\% & 75.00\% & 76.44\%\\
    64.27\% & 31.37\% & 89.51\% & 66.53\%\\
    60.57\% & 29.97\% & 65.87\% & 44.45\%\\
    \midrule
    1,529K & 405K & 1,529K & 405K\\
    \bottomrule
  \end{tabular}   
  \caption{Percentage of simple RPQs and 2RPQs in the
    Wikidata query logs in the study \citep{BonifatiMT-www19} (left) and the diverse query
    logs of \citep{BonifatiMT-vdlbj20} (right). For every analysis, we show
    percentages on all valid queries (Valid) and on all valid queries after
    duplicate elimination (Unique).}\label{tab:percentages}
\end{table*}

\paragraph{Organization}
In Section~\ref{sec:preliminaries} we introduce the necessary notation. In
Section~\ref{sec:mainresults} we present our main results which are then proved
in Sections~\ref{sec:notransitive}--\ref{sec:wfragment}. We discuss related work
in detail in Section~\ref{sec:relatedwork} and we conclude in
Section~\ref{sec:conclusions}.
Due to the page limit, we can only provide
sketches of some of the proofs. We will make longer proofs available on ArXiv.
\wim{OK like this? Do we already want to add a link?}

\makeatletter{}
\section{Preliminaries}\label{sec:preliminaries}

Let $\Sigma$ be an infinite set of \defstyle{labels}, to which we sometimes also
refer as the \defstyle{alphabet}. We abstract knowledge bases (or KBs, knowledge
graphs, or graph databases) as finite, edge-labeled directed graphs $K=(V,E)$,
where $V$ is a finite nonempty set of nodes, and $E$ is a set of labeled
directed edges $(u,a,v)\in V \times \Sigma \times V$.
A \defstyle{path} is a (possibly empty) sequence $\pi = (v_0,a_1,v_1)\cdots
(v_{n-1},a_n,v_n)$ of edges; we say that $\pi$ is a path from $v_0$ to $v_n$.
The \defstyle{length} of $\pi$ is the number $n \geq 0$ of edges in the sequence. We
denote by $\lab(\pi) $ the word $a_1\cdots a_n$ of edge labels seen along the path.
If all edges of $\pi$ have the same label $a \in \Sigma$, we say $\pi$ is an
\defstyle{$a$-path}. By $\varepsilon$ we denote the empty word. 
Regular expressions are defined as usual. We use uppercase letters $R$ for regular
expressions and denote their language by $L(R)$.

A \defstyle{conjunctive regular path query $(\crpq)$} has the general form
$Q(x_1, \ldots, x_n) \leftarrow A_1 \wedge\ldots \wedge A_m$. The
\defstyle{atoms} $A_1, \ldots, A_m$ are of the form $y R z$, where $y$ and $z$
are variables and $R$ is a regular expression. Each \defstyle{distinguished variable $x_j$}
from the left hand side has to occur in some atom on the right hand side. A
\defstyle{homomorphism} from $Q$ to $K$ is a mapping $\mu$ from the variables of $Q$
to $V$. Such a homomorphism \defstyle{satisfies} an atom $x R y$ if there is a path from
$\mu(x)$ to $\mu(y)$ in $K$ which is labeled with a word in $L(R)$.
A 
homomorphism from $Q$ to $K$ is called a \defstyle{satisfying homomorphism} if it
satisfies each atom $A_i$. For brevity, we also use the term
\defstyle{embedding} for satisfying homomorphisms.
The set of \defstyle{answers $\ans(Q,K)$} of a \crpq
$Q$ over a knowledge base $K$ is the set of tuples $(d_1, \dotsc, d_n)$ of nodes
of $K$ such that there exists a satisfying homomorphism for $Q$ on $K$ that maps
$x_i$ to $d_i$ for every $1 \leq i \leq n$.

Given two \crpqs $Q_1$, $Q_2$, we say that $Q_1$ is \defstyle{contained} in
$Q_2$, denoted by $Q_1 \subseteq Q_2$, if $ans(Q_1,K) \subseteq ans(Q_2,K)$ for
every knowledge base $K$. We say $Q_1$ is \defstyle{equivalent} to $Q_2$,
denoted by $Q_1 \equiv Q_2$, if $Q_1 \subseteq Q_2$ and $Q_2 \subseteq Q_1$.
We study the following problem, for various fragments $\+F_1,\+F_2$ of $\crpq$.
\decisionproblem{.45\textwidth}{
	\querycon of $\+F_1$ in $\+F_2$}{Two queries $Q_1 \in \+F_1$, $Q_2 \in \+F_2$.}{Is $Q_1 \subseteq Q_2$? }

\paragraph{Example.} To illustrate query containment we consider the
following example. Let $Q_1(x_1, x_2) \leftarrow (x_1\; \textsf{app}\; jm_1)
\land (x_2 \; \textsf{app} \; jm_1) \land (jm_1 \; \textsf{app}
  \; jm_2)$. Query $Q_1$ returns $(x_1$, $x_2)$ only if they were both the apprentices
of $jm_1$ (a Jedi master) who was in turn an apprentice of $jm_2$. Now consider
$Q_2(x_1, x_2) \leftarrow (x_1 \; \textsf{app} \cdot \textsf{app} \; jm)
 \land (x_2\; \textsf{app}\cdot \textsf{app} \; jm)$. We see that $Q_1
\subseteq Q_2$. However if we remove the last atom from $Q_1$, $Q_1 \subseteq
Q_2$ is not necessarily true. The following database provides a counterexample.
    \begin{figure}[h]
	\centering
	\begin{subfigure}{\linewidth}\centering
	\begin{tikzpicture}[auto, >=latex, 	triangle/.style = { regular polygon, regular polygon sides=3, inner sep=0, minimum size=1cm}]	\def\stretch{1.5};
	\def\ystretch{0.3};
	
	\node  (x0) at (\stretch*0,\ystretch*0) {};
	\node  (x1) at (\stretch*1,\ystretch*-1.0) {};
	\node (x2) at (\stretch*1,\ystretch*1.0) {};
	
	\fill (x0) circle (2pt) node[label=left:$\textsc{Yoda}$,draw]{};
	\fill (x1) circle (2pt) node[label=right:$\textsc{Luke}$,draw]{};
	\fill (x2) circle (2pt) node[label=right:$\textsc{Obi-Wan}$,draw]{};
	
	\path 
	(x1) edge	[->,below]  node  {$\textsf{app}$} (x0)
	(x2) edge	[->,above]  node  {$\textsf{app}$} (x0)
	;
	\end{tikzpicture}
	\end{subfigure}
	\end{figure}
	$Q_1$ without the last atom returns $(\textsc{Luke}$, $\textsc{Obi-Wan}$) though
$Q_2$ does not. \hfill $\blacksquare$ 
Let $Q$ be the \crpq $Q(x_1, \ldots, x_n) \leftarrow y_1 R_1 y_2 \wedge \ldots
\wedge y_{2m-1} R_m y_{2m}$. Let $K$ be a knowledge base and $\nu$ a total mapping from
the variables $\{x_1, \ldots, x_n, y_1, \ldots, y_{2m}\}$ of $Q$ to the nodes of
$K$. Then $K$ is \defstyle{$\nu$-canonical} for $Q$ if
\begin{itemize}
\item  $K$ constitutes of $m$ simple paths,
          one for each atom of $Q$, which are node- and edge-disjoint except for the
  start and end nodes, and 
\item for each $i \in \{1,\ldots, m\}$ the simple path $\pi_i$ associated to the
  atom $y_{2i-1} R_i y_{2i}$ connects the node $\nu(y_{2i-1})$ to the node
  $\nu(y_{2i})$ and has $\lab(\pi_i) \in L(R_i)$.
\end{itemize}
It is easy to see that $Q_1 \not\subseteq Q_2$ iff there exists a knowledge base $K$ and a mapping $\nu$ from the variables of $Q_1$ to the nodes of $K$ such that (i) $K$ is $\nu$-canonical for $Q_1$ and (ii) $(\nu(x_1), \ldots, \nu(x_n))\notin \ans(Q_2,K)$.
Therefore, to decide \querycon, it suffices to study containment on knowledge bases  which are $\nu$-canonical for $Q_1$. We call these knowledge bases  \defstyle{canonical models} of $Q_1$.

It is well-known that there is a natural correspondence between (the bodies of)
CRPQs and graphs by viewing their variables as nodes and the atoms as edges. We
will therefore sometimes use terminology from graphs for CRPQs (e.g., connected
components).

\makeatletter{}\section{Main Results}\label{sec:mainresults}

\begin{table}  \centering 
	\begin{tabular}{@{}l@{\hspace{2.5mm}}l@{\hspace{2.5mm}}l@{\hspace{2.5mm}}l@{}}
		\toprule
		${\mathcal F}$ & ${\mathcal F} \subseteq {\mathcal F}$ & ${\mathcal F} \subseteq \crpq$ &  $\crpq \subseteq {\mathcal F}$\\
		\midrule
		$a$ & \np($\dag$) & \np (\ref{theo:a-in-crpq}) &\pitwo(\ref{theo:crpq-in-a}) \\
		$A$ & \pitwo (\ref{theo:A-in-a}) & \pitwo &   \pspace(\ref{theo:crpq-in-A}) \\
		$(a,a^*)$ & \pitwo ($\ddag$) & \pitwo& \pspace(\ref{theo:crpq-in-a-astar})\\
		$(A,a^*)$ & \pitwo & \pitwo(\ref{cor:A-astar-in-crpq}) & \pspace(\ref{theo:crpq-in-A-astar})\\
		$(a,A^*)$  &\expspace (\ref{theo:a-Astar}) &\expspace& \expspace \\
		$(A,A^*)$ &\expspace  & \expspace ($\star$) & \expspace ($\star$) \\
		\bottomrule
	\end{tabular}
	\caption{Complexity of Containment of different fragments ${\mathcal F}$ of
    \crpqs. All results are complete for the class given.
		We provide references in round brackets. When there is no bracket,
    the result follows directly from another cell in the table.
		($\dag$): \citep{ChandraM-stoc77}, ($\ddag$): \cite[fragment $(l^*)$]{DeutschT-dbpl01},
    ($\star$): \citep{CalvaneseGLV-kr00}}
	\label{tab:complexity-smaller}
\end{table}

For a class of regular languages $\cal L$ we write $\crpq(\cal L)$ to denote the
set of CRPQs whose languages (of regular expressions in atoms) are in $\cal L$.
We use the same abbreviations for ${\cal L}$ as discussed in the Introduction: $a$ for regular
expressions that are just a single symbol, $a^*$ for Kleene closures of a single
symbols, $A$ for disjunctions (or sets) of symbols, and $A^*$ for Kleene
closures of disjunctions (or sets) of symbols. A sequence of abbreviations in
$\cal L$ represents options: for instance, $\crpq(a, A^*)$ is the set of CRPQs 
in which each atom uses either a single symbol or a transitive closure
of a disjunction of symbols.\footnote{In some proofs, we
  also allow concatenations of these forms. But this does not make a difference:
  in CRPQs such concatenations can always be eliminated at the cost of a few
  extra variables.}

In this paper, we give a complete overview of the complexity of containment for
the fragments $\cF = \crpq(a)$, $\crpq(A)$, $\crpq(a,a^*)$, $\crpq(A,a^*)$,
$\crpq(a,A^*)$, and $\crpq(A,A^*)$. That is, for each of these fragments we prove
that their containment problem is complete for \np, $\pitwo$, or \expspace.
Furthermore, for each of these fragments $\cF$, we give a complete overview of
the complexity of the containment problems of the form $\cF \subseteq \crpq$ and
$\crpq \subseteq \cF$. An overview of our results can be found in
Table~\ref{tab:complexity-smaller}. All results are completeness results. Some of the results were already obtained in other papers, which we
indicate in the table. 

Interestingly, our results imply that containment is \expspace-complete only if
we allow sets of symbols under the Kleene star both in the left- and right-hand
queries. As soon as we further restrict the usage of the Kleene star on one
side, the complexity drops to \pspace or even \pitwo. As it turns out, queries
having $a^*$ as only means of recursion is still very representative of the
queries performed in practice, as evidenced in Table~\ref{tab:percentages},
where over 98\% of the RPQs in the Wikidata logs are of this form. In the
DBpedia$^\pm$ logs, this percentage is still around 70\% of the total RPQs. Two
main reasons why this percentage is lower here are that ``wildcards'' of the
form $!a$, i.e., follow an edge \emph{not} labeled $a$, and 2RPQs of the form
$(a + \hat{\ }a)^*$, i.e., undirected reachability over $a$-edges, make up
around 15\% and 20\% respectively of the expressions in unique queries in
DBpedia$^\pm$. The fact that equivalence testing is \pitwo for these queries,
gives hope that optimizations by means of static analysis may be practically
feasible for most of the CRPQ used for querying ontologies and RDF data.

Our results apply to both finite and infinite sets of labels, if we do not explictly
say otherwise. The reason is that as long as the query language does not allow
for wildcards, we can always restrict to the symbols explicitly used in the
queries, which is always a finite set.

If wildcards are allowed, the complexity of query containment can heavily depend
on the finiteness of the alphabet of edge labels $\Sigma$. We discovered that our techniques can be used
to settle an open question (and correct an error) in the work of
\citet{DeutschT-dbpl01}, who have also considered containment of simple CRPQs.
Deutsch and Tannen considered CRPQ fragments motivated by the navigational
features of XPath and claimed that containment for their \emph{W-fragment} (see
Section~\ref{sec:wfragment} for a definition), using infinite alphabets, is \pspace-hard.
However, we prove that containment for this fragment is in \pitwo
(Theorem~\ref{theo:w-in-crpq}). The minor error is that Deutsch and Tannen assumed \emph{finite}
alphabets in their hardness proof. In fact, when one indeed assumes a finite set
of edge labels in KBs, we prove that the containment problem for
the W-fragment is \expspace-complete (Proposition~\ref{prop:w}).

\makeatletter{}\section{No Transitive Closure}\label{sec:notransitive}

In this section we study simple CRPQ fragments without transitive closure.
We first observe that $\crpq(a)$ is equivalent to the well-studied class of
conjunctive queries (CQ) on binary relations.
\begin{theorem}[\citeauthor{ChandraM-stoc77} \citeyear{ChandraM-stoc77}]
  Containment of $\crpq(a)$ in $\crpq(a)$ is \np-complete.
\end{theorem}
Even when we allow arbitrary queries on the right, the complexity stays the
same. The reason is that the left query has a single canonical model $K$ of
linear size, and
thus we can check containment by testing for a satisfying homomorphism from $Q_2$ to
$K$ (that preserves the distinguished nodes).
\begin{theorem}\label{theo:a-in-crpq}
  Containment of $\crpq(a)$ in $\crpq$ is \np-complete
\end{theorem}

If we allow more expressive queries on the left, the complexity becomes
$\pitwo$, even if the right-hand queries are CQs.
\begin{theoremrep} \label{theo:A-in-a}
	$\querycon$ of $\crpq(A)$ in $\crpq(a)$ is \pitwo-complete, even if the size of the alphabet is fixed.
\end{theoremrep}
\begin{proofsketch}
	The upper bound is immediate from Corollary~\ref{cor:A-astar-in-crpq}, which in turn follows from Theorem~\ref{theo:w-in-crpq}. Both these results are proved later. For
  the lower bound, we reduce from $\forall \exists$-QBF (i.e.,
  $\Pi_2$-Quantified Boolean Formulas).
Let
	\[
	\Phi\;\; =\;\; \forall x_1, \ldots, x_n\; \exists y_1, \ldots, y_\ell\;
	\varphi(x_1, \ldots, x_n,y_1, \ldots, y_\ell)
	\]
	be an instance of
	$\forall \exists$-QBF such that $\varphi$ is quantifier-free and in
	3-CNF. We construct boolean queries $Q_1$ and $Q_2$ such that
	$Q_1 \subseteq Q_2$ if, and only if, $\Phi$ is satisfiable.
	
  The query $Q_1$ is defined in Figure~\ref{fig:queryA:Q1-DE}, over the alphabet of labels $\{a,x_1,\dots,\allowbreak x_n, \allowbreak y_1,\dots, \allowbreak y_\ell, \allowbreak t,f\}$. We now explain how we define $Q_2$, over the same alphabet. Every clause of $\Phi$ is represented
  by a subquery in $Q_2$, as depicted in Figure~\ref{fig:queryA:q2}. All nodes
  with identical label ($y_{1,t}$ and $y_{1,f}$ in gadgets $D,
  E$) in Figures~\ref{fig:queryA:Q1-DE} and~\ref{fig:queryA:q2} are the same
  node. (So, both queries are DAG-shaped.) Note that for every clause and every existentially quantified literal $y_i$ therein we have one node named $y_{i,tf}$ in $Q_2$. The $E$-gadget is designed
  such that every represented literal can be homomorphically embedded, while 
  exactly one literal has to be embedded in the $D$-gadget.

  The intuitive idea is that the valuation of the $x$-variables is given by the
  concrete canonical model $K$ (i.e., whether the corresponding edge is labeled
  $t$ or $f$ in the $D$ gadget), while the valuation of the $y$-variables is
  given by the embedding of $Q_2$ into $K$ (i.e., whether the corresponding node
  is embedded into the node $y_{\mbox{\textvisiblespace},t}$ or
  $y_{\mbox{\textvisiblespace},f}$). The embedding of $y$-variables across several
  clauses has to be consistent, as all clauses share the same nodes
  $y_{\mbox{\textvisiblespace},tf}$, which uniquely get embedded either into 
  $y_{\mbox{\textvisiblespace},t}$ or $y_{\mbox{\textvisiblespace},f}$. Hence, when the
  formula $\Phi$ is satisfiable, for any assignment to the variables $\{x_i\}$
  (given by the choice of $t$/$f$ edges in $D$), there is a mapping from
  $y_{\mbox{\textvisiblespace},tf}$ to one of $y_{\mbox{\textvisiblespace},f}$ or
  $y_{\mbox{\textvisiblespace},t}$. This gives $Q_1 \subseteq Q_2$. Conversely, if
  $Q_2$ can be embedded in $K$, then, for a choice of $t$/$f$ edges in $D$, we
  have an embedding of each clause gadget of $Q_2$ in $K$. In particular, we can
  always map a literal in each clause of $Q_2$ to $D$, ensuring that $\varphi$ is
  satisfied. As this is true for any knowledge base $K$ obtained for all
  possible $t$/$f$ assignments to $\{x_i\}$, we obtain $\Phi$ is satisfiable.
  
  We note that this result can be extended to alphabets of constant size by
  encoding $x_i$ as $\hat x_i= \tileN^{i-1}\tileY\tileN^{n-i-1} \in
  \{\tileN, \tileY\}^n$ and $y_i$ as $\hat y_i= \tilen^{i-1}\tiley\tilen^{\ell-i-1} \in
  \{\tilen, \tiley\}^\ell$.
\end{proofsketch}
\begin{proof}
	The upper bound follows immediately from Deutsch and Tannen~\citep{DeutschT-dbpl01},
	to be more precise, from their problem named $(*,\mid)$.
	
	For the lower bound we use a reduction from
	$\forall \exists$-QBF. The main idea is to use
	sets $\{t,f\}$ in $Q_1$ to encode true or false. 
	
	More precisely,  let
	\[
	\Phi\quad =\quad \forall x_1, \ldots, x_n\; \exists y_1, \ldots, y_\ell\;
	\varphi(x_1, \ldots, x_n,y_1, \ldots, y_\ell)
	\]
	be an instance of
	$\forall \exists$-QBF such that $\varphi$ is quantifier free and in
	3-CNF. We construct boolean queries $Q_1$ and $Q_2$ such that
	$Q_1 \subseteq Q_2$ if and only if $\Phi$ is satisfiable.

	The \textbf{query $Q_1$} is sketched in
	Figure~\ref{fig:queryA:Q1-DE} and built as follows: The basis is an $a$-path of
	length 4. We add 4 gadgets $E$ to the outer
	nodes of the path and one gadget $D$ at the innermost.  The choice of 4 $E$ gadgets 
	surrounding the $D$ gadget will be made clear once we discuss $Q_2$. 
		Basically,
	the $E$-gadgets will accept everything while the $D$-gadget will
	ensure that the chosen literal evaluates to true. The gadgets are also depicted 
	in Figure~\ref{fig:queryA:Q1-DE}. The gadgets are constructed as follows.

	The \textbf{gadget $D$} is constructed such that the root node has
	one outgoing edge for each variable in $\Phi$, that is, $n+\ell$
	many. Each edge is labeled differently, that is,
	$x_1, \ldots, x_n, y_1, \ldots, y_\ell$. After each $x_i$-edge we
	add a $\{t,f\}$-edge. Each of them
	leads to a different node.  For each $i \in \{1,\ldots, \ell\}$ we do
	the following. We add a $t$-edge to a node we name \yit after the
	$y_i$-edge and an edge labeled $f$ that leads to a node we name
	\yif.  We named these nodes because we need those nodes also in the
	$E$-gadgets. Nodes with the same names across gadgets are actually the same node. 
	
	Each \textbf{gadget $E$} is constructed similar to the $D$
	gadget. The root node has one outgoing
	edge for each variable in $\Phi$, that is $n+\ell$ many. Each edge
	is labeled differently, that is
	$x_1, \ldots, x_n, y_1, \ldots, y_\ell$. After each $x_i$-edge we
	add a $t$-edge and an $f$-edge. Each of those edges leads to a different node.  After each $y_i$-edge we add a $t$-edge and a $f$-edge to both \yit and to \yif.
	
	We now explain the construction of $Q_2$. An example is given in Figure~\ref{fig:queryA:q2}.
	For each clause $i$, \textbf{query $Q_2$} has a small DAG, 
	which might share nodes ($\yktf$) with the DAGs constructed for the other clauses.  
	For clause $i$, we construct
	$C_i^1$, with an  $a$-edge to the gadget $C_i^2$, and from there again a 
	$a$-edge to the gadget $C_i^3$.
	
	The gadget $C_i^j$ represents the $j$th literal in the $i$th
	clause. Since the QBF is in 3-CNF, we have $j \in \{1,2,3\}$.
	If the literal is the positive variable $x_k$,
	$C_i^j$ is a path labeled $x_k t$. If it is the
	negative variable $\neg x_k$, $C_i^j$ is a path labeled
	$x_k f$. If the
	literal is the positive variable $y_k$, $C_i^j$ is a path
	labeled $y_k t$ and it ends in a node we call $\yktf$ and, if it
	is the negative variable $\neg y_k$, $C_i^j$ is a path
	labeled $y_k f$ and it ends in $\yktf$, too. 
	
	This completes the construction. We will now give some intuition.
	The gadget $D$ controls via the $\{t,f\}$-edges, which variables
	$x_i$ are set to true and which to false. We will consider it true
	whenever there is an $x_i t$-path and false otherwise, that is,
	there is an $x_i f$-path.  
	Depending on this, we can either map
	$C_i^j$ into it or not. The $E$ gadgets are constructed such
	that each $C_i^j$ can be mapped into it. The query $Q_2$ can decide, which 
	path should be mapped into $D$ and therefore, which literal should be verified. 
	The structure of $Q_1$ where two $E$ gadgets each surround the $D$ gadget aids in embedding the clauses
	$C_i^1, C_i^2, C_i^3$ for each $i$ in the knowledge base $G$. If the $i$th clause is $(x_2 \vee \neg y_1 \vee \neg x_3)$ 
	and if in the canonical model $G$, we have the assignment of $f$ to $x_2$, $t$ to $x_3$, then we can embed 
	$C_i^1, C_i^3$ in the second and third $E$'s, and $\neg y_1$ can be embedded in $y_{1f}$ in $D$. Embedding 
	$\neg y_1$  in $y_{1f}$ fixes the assignment $f$ to $y_1$ across all 
	gadgets $E, D$, and all clauses in $Q_2$. Likewise, for a clause   
	 $(x_1 \vee \neg x_4 \vee y_5)$ in $\Phi$, and an assignment $f$ to $x_1$, 
	 $t$ to $x_4$ in the canonical model $G$, we can embed $x_1, \neg x_4$ in the first and second $E$'s and 
	 $y_5$ to the node $y_{5t}$.

	We will now show
	correctness, that is: $Q_1 \subseteq Q_2$ if and only if $\Phi$ is
	satisfiable.  Let $Q_1 \subseteq Q_2$. Then there exists
	a homomorphism from $Q_2$ to each canonical model of $Q_1$. The
	canonical models of $Q_1$ look exactly like $Q_1$ except that each
	$\{t,f\}$-edge is replaced with either $t$ or $f$.

	Let $B$ be an arbitrary canonical model of $Q_1$ and
	$D_B$ the gadget $D$ in $B$. 
	We define $\theta_B(x_i)=1$ if $D_B$ contains an
	$x_i t$-path and $\theta_B(x_i)=0$ otherwise.
	Let $h$ be a homomorphism mapping $Q_2$ to $B$.
	We furthermore define $\theta_B(y_i)=1$ if $h$ maps
	\yitf to \yit and $\theta_B(y_i)=0$ otherwise, i.e., if
	\yitf is mapped to \yif.  We now show that $\theta_B$ is well-defined and
	satisfies $\varphi$.  It is obvious that each $C_i^j$ will be mapped either to
	the gadget $D_B$ or to $E$ and that for each $i \in \{1,\ldots, m\}$
	exactly one $C_i^j$ is mapped to $D_B$.  
	If $C_i^j$ corresponds to $x_k$, i.e.,
	it is a path labeled $x_k t$, then it can only be mapped into $D_B$
	if $\theta_B(x_k)=1$. Analogously, if $C_i^j$ corresponds to
	$\neg x_k$, it is a path labeled $x_k f$, and can
	therefore only be mapped into $D_B$ if $\theta_B(x_k)=0$.  If
	$C_i^j$ corresponds to $y_k$ or $\neg y_k$, it can always be
	mapped into $D_B$, but since \yktf can only be mapped either to $y_{k,t}$
	or $y_{k,f}$, we can either map positive $y_k$ into $D_B$ or negative
	ones, but not both. Therefore, the definition of $\theta_B(y_k)$ is
	unambiguous and it indeed satisfies $\varphi$.
	
	Since $B$ is arbitrary, we obtain a choice $y_1, \ldots, y_\ell$ for
	all possible truth-assignments to $x_1,\ldots, x_n$ this
	way. Therefore, $\Phi$ is satisfiable.

	For the only if direction let $\Phi$ be satisfiable. Then we find
	for each truth-assignment to $x_1, \ldots, x_n$ an assignment to
	$y_1, \ldots, y_\ell$ such that
	$\varphi(x_1, \ldots, x_n, y_1, \ldots, y_\ell)$ is true. Let
	$\theta$ be a function that, given the $x_i$, returns an assignment
	for all $y_i$ such that the formula evaluates to true. We will show
	how to map $Q_2$ into an arbitrary canonical database $B$ of $Q_1$.
					
	Let $B$ and $\theta$ be given. Let $D_B$ be again the gadget $D$ in $B$.
	We use $\theta$ to obtain truth-values for
	$y_1, \ldots, y_\ell$ as follows.  Since this assignment is satisfiable, there
	is a literal that evaluates to true in each clause. We map this
	literal to $D_B$ and the others in this clause to gadgets $E$.  If this
	literal is $x_i$, then we can map to the $x_i t$-path in $D_B$. If it is $\neg x_i$, then we can map to the
	$x_i f$ path in $D_B$.  If the
	literal is $y_i$, we can map the $y_i t$-path ending in $\yitf$ to
	$D_B$. This also implies that each $\yitf$ in $Q_2$ is mapped to \yit,
	which is no problem since each path mapped to $E$ can choose freely
	between \yit and \yif and, since $\theta$ is a function, we only
	have either $\theta(y_i)=1$ or $\theta(y_i)=0$.  Analogously, if the
	literal is $\neg y_i$, we can map the $y_i f$-path ending in
	$\yitf$ to $D_B$, which implies that each $\yitf$ in $Q_2$ is mapped to
	\yif.
\end{proof}

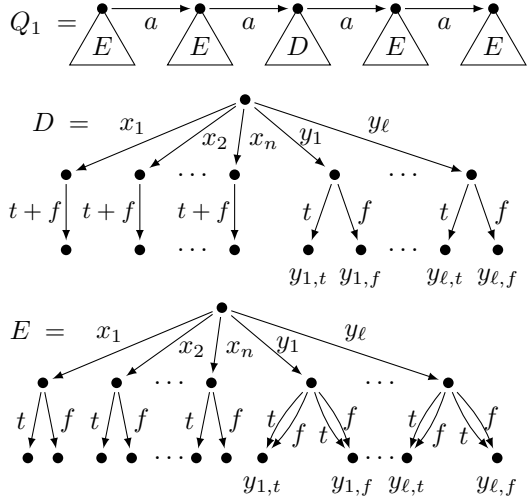
\begin{figure}[t]
	\centering
	\begin{subfigure}{\linewidth}\centering
	\begin{tikzpicture}[auto, >=latex, 	triangle/.style = { regular polygon, regular polygon sides=3, inner sep=0, minimum size=1cm}]	\def\stretch{1.3};
	\def\ystretch{1};
	\node at (-.6*\stretch, -1.7*\ystretch) {$Q_1\ =$};
	
	\node  (v6) at (\stretch*0,\ystretch*-1.5) {};
	\node  (v7) at (\stretch*1,\ystretch*-1.5) {};
	\node (v8) at (\stretch*2,\ystretch*-1.5) {};
	\node (v9) at (\stretch*3,\ystretch*-1.5) {};
	\node  (v10) at (\stretch*4,\ystretch*-1.5) {};
	\node [draw,triangle,inner sep=-1]  (e1) at (\stretch*0,\ystretch*-2) {$E$};
	\node [draw,triangle,inner sep=-1]  (e2) at (\stretch*1,\ystretch*-2) {$E$};
	\node [draw,triangle,inner sep=-1]  (e3) at (\stretch*2,\ystretch*-2) {$D$};
	\node [draw,triangle,inner sep=-1]  (e4) at (\stretch*3,\ystretch*-2) {$E$};
	\node [draw,triangle,inner sep=-1]  (e5) at (\stretch*4,\ystretch*-2) {$E$};
	
	\foreach \i in {6,7,8,9,10}{
		\fill (v\i) circle (2pt);
	}
	\path 
	(v6) edge	[->,below]  node  {$a$} (v7)
	(v7) edge	[->,below]  node  {$a$} (v8)
	(v8) edge	 [->,below]    node  {$a$} (v9) 
	(v9) edge	 [->,below]    node  {$a$} (v10)
	;
	\end{tikzpicture}
	\end{subfigure}
	\begin{subfigure}{\linewidth}\centering
		\begin{tikzpicture}[auto,>=latex]
				\def\stretch{1.4};
		\def\ystretch{1};
		\node at (-.75*\stretch, -.7*\ystretch) {}; 		\node at (-.75*\stretch, -1.3*\ystretch) {$D\ =$};
				\node [] (v2) at (\stretch*1,\ystretch*-1) {};
		\node [] (g1) at (\stretch*-.7,\ystretch*-2) {};
		\node [] (g2) at (\stretch*0,\ystretch*-2) {};
		\node [] (g3) at (\stretch*.5,\ystretch*-2) {$\ldots$};
		\node [] (g4) at (\stretch*.9,\ystretch*-2) {};
				
		\fill (v2) circle (2pt);
		\fill (g1) circle (2pt);
		\fill (g2) circle (2pt);
		
		\fill (g4) circle (2pt);
		\path (v2) edge	[->,left]  node [yshift=4pt] {$x_1$} (g1);
		\path (v2) edge	[->,right]  node [yshift=-1.5pt] {$x_2$} (g2);
		\path (v2) edge	[->,right]  node [yshift=-1.5pt] {$x_n$} (g4);
				\node [] (b1) at (\stretch*-.7,\ystretch*-3) {};
		\node [] (b2) at (\stretch*0,\ystretch*-3) {};
		\node [] (b3) at (\stretch*.5,\ystretch*-3) {$\ldots$};
		\node [] (b4) at (\stretch*.9,\ystretch*-3) {};
		
								\foreach \i in {1,2,4}{
			\fill (b\i) circle (2pt);
						\path (g\i) edge	[->,left]  node[xshift=2pt]  {\small $t+f$} (b\i);
					}
		
				\node [] (y1) at (\stretch*1.85,\ystretch*-2) {};
				\node [] (y3) at (\stretch*2.5,\ystretch*-2) {$\ldots$};
		\node [] (y4) at (\stretch*3.15,\ystretch*-2) {};
				\foreach \i in {1,4}{
			\fill (y\i) circle (2pt);
		}
		\path (v2) edge	[->,right]  node  {$y_1$} (y1);
				\path (v2) edge	[->]  node [yshift=-2pt] {$y_\ell$} (y4);
				
		\node [label=below:{$y_{1,t}$}] (x1) at (\stretch*1.6,\ystretch*-3) {};
		\node [label=below:{$y_{1,f}$}] (x2) at (\stretch*2.1,\ystretch*-3) {};
		\node [] (x3) at (\stretch*2.5,\ystretch*-3) {$\ldots$};
		\node [label=below:{$y_{\ell,t}$}] (x4) at (\stretch*2.9,\ystretch*-3) {};
		\node [label=below:{$y_{\ell,f}$}] (x5) at (\stretch*3.4,\ystretch*-3) {};
		\foreach \i in {1,2,4,5}{
			\fill (x\i) circle (2pt);
		}
		\path (y1) edge	[->,left]  node  {$t$} (x1);
		\path (y1) edge	[->,right]  node  {$f$} (x2);
		\path (y4) edge	[->,left]  node  {$t$} (x4);
		\path (y4) edge	[->,right]  node  {$f$} (x5);
										\end{tikzpicture}
	\end{subfigure}
	\begin{subfigure}{\linewidth}\centering
		\begin{tikzpicture}[auto,>=latex]
				\def\stretch{1.4};
		\def\ystretch{1};
		\node at (-.75*\stretch, -1.3*\ystretch) {$E\ =$};
				\node [] (v2) at (\stretch*1,\ystretch*-1) {};
		\node [] (g1) at (\stretch*-.7,\ystretch*-2) {};
		\node [] (g2) at (\stretch*0,\ystretch*-2) {};
		\node [] (g3) at (\stretch*.5,\ystretch*-2) {$\ldots$};
		\node [] (g4) at (\stretch*.9,\ystretch*-2) {};
				
		\fill (v2) circle (2pt);
		\fill (g1) circle (2pt);
		\fill (g2) circle (2pt);
		
		\fill (g4) circle (2pt);
		\path (v2) edge	[->,left]  node [yshift=4pt] {$x_1$} (g1);
		\path (v2) edge	[->,right]  node [yshift=-1.5pt] {$x_2$} (g2);
		\path (v2) edge	[->,right]  node [yshift=-1.5pt] {$x_n$} (g4);
				\node [] (b1) at (\stretch*-.7-0.2,\ystretch*-3) {};
		\node [] (b2) at (\stretch*0-0.2,\ystretch*-3) {};
		\node [] (b3) at (\stretch*.5,\ystretch*-3) {$\ldots$};
		\node [] (b4) at (\stretch*.9-0.2,\ystretch*-3) {};
		
		\node [] (n1) at (\stretch*-.7+0.2,\ystretch*-3) {};
		\node [] (n2) at (\stretch*0+0.2,\ystretch*-3) {};
		\node [] (n4) at (\stretch*.9+0.2,\ystretch*-3) {};
		\foreach \i in {1,2,4}{
			\fill (b\i) circle (2pt);
			\fill (n\i) circle (2pt);
			\path (g\i) edge	[->,left]  node  {$t$} (b\i);
			\path (g\i) edge	[->,right]  node  {$f$} (n\i);
		}
		
				\node [] (y1) at (\stretch*1.85,\ystretch*-2) {};
				\node [] (y3) at (\stretch*2.5,\ystretch*-2) {$\ldots$};
		\node [] (y4) at (\stretch*3.15,\ystretch*-2) {};
				\foreach \i in {1,4}{
			\fill (y\i) circle (2pt);
		}
		\path (v2) edge	[->,right]  node  {$y_1$} (y1);
				\path (v2) edge	[->]  node [yshift=-2pt] {$y_\ell$} (y4);
				
		\node [label=below:{$y_{1,t}$}] (x1) at (\stretch*1.6-.3,\ystretch*-3) {};
		\node [label=below:{$y_{1,f}$}] (x2) at (\stretch*2.1+.2,\ystretch*-3) {};
		\node [] (x3) at (\stretch*2.5,\ystretch*-3) {$\ldots$};
		\node [label=below:{$y_{\ell,t}$}] (x4) at (\stretch*2.9-.2,\ystretch*-3) {};
		\node [label=below:{$y_{\ell,f}$}] (x5) at (\stretch*3.4+.3,\ystretch*-3) {};
		\foreach \i in {1,2,4,5}{
			\fill (x\i) circle (2pt);
		}
		\path (y1) edge	[->,left, bend right = 10]  node [xshift=1.5pt] {$t$} (x1);
		\path (y1) edge	[->,right, bend left = 10]  node [xshift=-4pt, yshift=-5pt] {$f$} (x1);
		\path (y1) edge	[->,left, bend right = 10]  node [xshift=4pt, yshift=-5pt] {$t$} (x2);
		\path (y1) edge	[->,right, bend left = 10]  node [xshift=-1.5pt]{$f$} (x2);
		\path (y4) edge	[->,left, bend right = 10]  node [xshift=1.5pt] {$t$} (x4);
		\path (y4) edge	[->,right, bend left = 10]  node [xshift=-4pt, yshift=-5pt]{$f$} (x4);
		\path (y4) edge	[->,left, bend right = 10]  node [xshift=4pt, yshift=-5pt]{$t$} (x5);
		\path (y4) edge	[->,right, bend left = 10]  node [xshift=-1.5pt]{$f$} (x5);
										\end{tikzpicture}
	\end{subfigure}
	\caption{Query $Q_1$ used in the proof of Theorem~\ref{theo:A-in-a} and the gadgets $D$ and $E$ used in $Q_1$.}
	\label{fig:queryA:Q1-DE}
\end{figure}
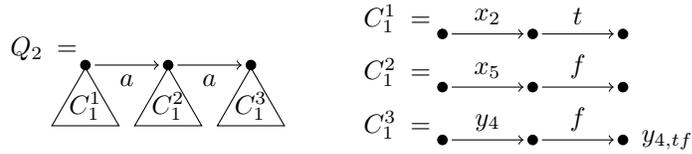
\begin{figure}[t]
	\begin{minipage}{.45\linewidth}
		\centering
		\begin{tikzpicture}[auto, 	triangle/.style = { regular polygon, regular polygon sides=3, inner sep=0, minimum size=1cm}]				\def\stretch{1.1};
		\def\ystretch{1.1};
		
		\node at (.5*\stretch, -.6*\ystretch) {$Q_2\ =$};
		\node [] (v16) at (\stretch*1,\ystretch*-.8) {};
		\node [] (v17) at (\stretch*2,\ystretch*-.8) {};
		\node [] (v18) at (\stretch*3,\ystretch*-.8) {};
		\node [draw,triangle,inner sep=-1]  (e1) at (\stretch*1,\ystretch*-1.3) {$C^1_1$};
		\node [draw,triangle,inner sep=-1]  (e2) at (\stretch*2,\ystretch*-1.3) {$C^2_1$};
		\node [draw,triangle,inner sep=-1]  (e3) at (\stretch*3,\ystretch*-1.3) {$C^3_1$};
		
		\foreach \i in {16,17,18}{
			\fill (v\i) circle (2pt);
		}
		\path 
		(v16) edge	[->,swap]  node  {$a$} (v17)
		(v17) edge	[->,swap]  node  {$a$} (v18)
		;
		\end{tikzpicture}
					\end{minipage}
	\begin{minipage}{.45\linewidth}
		\centering
		\begin{subfigure}{\textwidth}
			\begin{tikzpicture}[auto]
						\def\stretch{1.2};
			\def\ystretch{1};
			\node at (-.5*\stretch, .2*\ystretch) {$C_1^1\ =$};
			\node [] (v1) at (\stretch*0,0) {};
			\node [] (v2) at (\stretch*1,0) {};
			\node [] (v3) at (\stretch*2,0) {};
			\foreach \i in {1,2,3}{
				\fill (v\i) circle (2pt);
			}
			\path 
			(v1) edge	[->]  node  {$x_2$} (v2)
			(v2) edge	 [->]    node  {$t$} (v3) 
			;
			\end{tikzpicture}
		\end{subfigure}
		\begin{subfigure}{\textwidth}
			\begin{tikzpicture}[auto]
						\def\stretch{1.2};
			\def\ystretch{1};
			\node at (-.5*\stretch, .2*\ystretch) {$C_1^2\ =$};
			\node [] (v1) at (\stretch*0,0) {};
			\node [] (v2) at (\stretch*1,0) {};
			\node [] (v3) at (\stretch*2,0) {};
			\foreach \i in {1,2,3}{
				\fill (v\i) circle (2pt);
			}
			\path 
			(v1) edge	[->]  node  {$x_5$} (v2)
			(v2) edge	 [->]    node  {$f$} (v3) 
			;
			\end{tikzpicture}
		\end{subfigure}	
		\begin{subfigure}{\textwidth}
			\begin{tikzpicture}[auto]
						\def\stretch{1.2};
			\def\ystretch{1};
			\node at (-.5*\stretch, .2*\ystretch) {$C_1^3\ =$};
			\node [] (v1) at (\stretch*0,0) {};
			\node [] (v2) at (\stretch*1,0) {};
			\node [label=right:{$y_{4,tf}$}] (v3) at (\stretch*2,0) {};
			\foreach \i in {1,2,3}{
				\fill (v\i) circle (2pt);
			}
			\path 
			(v1) edge	[->]  node  {$y_4$} (v2)
			(v2) edge	 [->]    node  {$f$} (v3) 
			;
			\end{tikzpicture}
		\end{subfigure}
	\end{minipage}
	\caption{Example of $Q_2$ in the proof of Theorem~\ref{theo:A-in-a}
		for the formula $\varphi = (x_2 \vee\neg x_5 \vee \neg y_4 )$.}
	\label{fig:queryA:q2}
\end{figure}

On the other hand, even if we now allow arbitrary \crpqs on the left,
containment remains in $\pitwo$.
\begin{theorem}\label{theo:crpq-in-a}
	Containment of $\crpq$ in $\crpq(a)$ is $\pitwo$-complete. 
\end{theorem}
\begin{proof}
  The lower bound is immediate from Theorem~\ref{theo:A-in-a}. For the upper
  bound, we provide a $\sigmatwo$ algorithm for non-containment, which yields
  the result. Let $Q_1 \in \crpq$, $Q_2 \in \crpq(a)$, and $\#$ be a symbol not
  appearing in $Q_1$ or $Q_2$. For every atom $A=x R y$ of $Q_1$ we guess words
  $u_A$ and $v_A$ of length $\leq |Q_2|$ such that $u_A \Sigma^* v_A \cap L(R)
  \neq \emptyset$ and $|u_A v_A|<2|Q_2|$ implies that $u_A v_A \in L(R)$. We guess
  a component $Q'_2$ of $Q_2$ and we check that
  \begin{enumerate}[(1)]
	\item $Q'_2$ cannot be embedded in $Q'_1$, where $Q'_1$ is the KB resulting
    from replacing each atom $A= x R y$ with the path $u_A \cdot s_\# \cdot
    v_A$, where $s_\# = \varepsilon$ if $|u_Av_A| < 2 |Q_2|$ and $s_\# = \#$
    otherwise; and
	\item for every atom $A = x R y$ of $Q_1$ such that $|u_Av_A|= 2|Q_2|$ there
    is $w \in u_A \Sigma^* v_A \cap L(R)$ such that $Q'_2$ cannot be embedded in
    $w$. This last test amounts to checking that either (i) $Q'_2$ is not
    homomorphically equivalent to a path or, otherwise, (ii) if $Q'_2$ is
    homomorphically equivalent to a path with label $\hat w$,
		we test $u_A \Sigma^* v_A \cap L(R) \cap (\Sigma^*\hat w \Sigma^*)^c \neq \emptyset$.
  \end{enumerate}
  If tests (1) and (2) succeed, we found a knowledge base into which $Q_1$ can
  be embedded, but not $Q_2$. Testing whether $Q'_2$ can be homomorphically
  embedded in $Q'_1$ is in \np as the size of $Q'_1$ is polynomial in $Q_1$ and
  $Q_2$. Test (2) is in \conp as we need to check for an embedding of $Q'_2$ for
  each atom of $Q_1$.
\end{proof}

Allowing disjunctions in the right query is rather harmless if we only need to
consider polynomial-size canonical models to decide containment correctly. Even
if such canonical models may become exponentially large, they can sometimes be
encoded using polynomial size, allowing for $\pitwo$ containment algorithms
(cf.\ Corollary~\ref{cor:A-astar-in-crpq}, Theorem~\ref{theo:w-in-crpq}).
However, if we have arbitrary queries on the left, these techniques do not work
anymore, to the extent that
the problem becomes \pspace-complete.

The following theorem can be regarded as a generalization of the result of
\citet{BjorklundMS-mfcs13} [Theorem 9] stating that
the inclusion problem between a DFA over an alphabet $\Sigma=\{a,b,c\}$ and a
regular expressions of the form $\Sigma^* a \Sigma^n b \Sigma^*$
is \pspace-complete. \begin{theorem}\label{theo:crpq-in-A}
	$\querycon$ of $\crpq$ in $\crpq(A)$ is \pspace-complete, even if the size of the alphabet
  is fixed.
\end{theorem}
\begin{proof}
	    	  The upper bound follows from Theorem~\ref{theo:crpq-in-A-astar}, which we
  prove later. For the lower bound we reduce from the corridor tiling problem, a
  well-known \pspace-complete problem \citep{Chlebus-jcss86}. An instance of this
  problem is a tuple $(T,H,V,\bar{i},\bar{f},n)$, where $T$ is the set of tiles,
  $H,V \subseteq T \times T$ are the horizontal and vertical constraints,
  encoding which tiles are allowed to occur next to each other and on top of
  each other, respectively, $\bar{i} =i_1\dots i_n \in T^n$ is the initial row,
  $\bar{f} = f_1\dots f_n \in T^n$ is the final row, and $n$ encodes the length
  of each row in unary. The question is whether there exists a tiling solution,
  that is, an $N \in \nat$ and a function $\tau: \set{1,\dotsc, N} \times \set{1,\dotsc, n} \to T$
  such that $\tau(1,1) \dotsb \tau(1,n) = \bar i$, $\tau(N,1) \dotsb \tau(N,n) =
  \bar f$ and all horizontal and vertical constraints are satisfied:
  $(\tau(i,j),\tau(i,j+1)) \in H$ and $(\tau(i,j),\tau(i+1,j)) \in V$ for every
  $i,j$ in range.

	The coding idea is that the query $Q_1$ is
  a string describing all tilings with correct start and end tiles, with no
  horizontal errors, and having rows of the correct length. The query $Q_2$
  describes vertical errors. Then we have $Q_1 \subseteq Q_2$ if and only if
  there exists no valid tiling, i.e., every tiling has an error.

  Let $(T,H,V,\bar{i},\bar{f},n)$ be a corridor tiling instance as defined
  before. From the original proof of \citet{Chlebus-jcss86}, it follows that the
  following restricted version of corridor tiling remains \pspace-complete. The
  set of tiles $T$ is partitioned into $T=T_1 \uplus T_2 \uplus T_3 $, such that
  each row in a solution must belong to $T_1^* T_2 T_1^* \cup T_1^* T_3T_3
  T_1^*$. The original proof furthermore implies, that (i) $(T_1 \times T_1)
  \cup (T_1 \times T_2) \cup (T_2 \times T_1) \subseteq H$; and (ii) for all
  $u,v \in T_3$ with $(u,v)\in H$ we have that $T_1 \times \{u\} \subseteq H$
  and $\{v\} \times T_1 \subseteq H$. This implies that our horizontal errors
  can only occur with $T_2$ or $T_3$ involved, so only once per row. Therefore,
  we construct a new set $\tilde H$ defined as follows: $\tilde H = H \cap (T_2
  \times T_1 \cup T_1 \times T_2 \cup T_3 \times T_3)$. This set is used in
  the definition of query $Q_1$.

	We encode tiles as follows: each tile $t_i$ has an encoding 
	$\widehat{t_i}$ given by $\tilen^{i-1} \tileY
  \tilen^{|T|-i-1} e_1 \cdots e_{|T|}$, 
  where $e_j = \tiley$ if $(t_i, t_j) \in V$
  and $e_j = \tilen$, otherwise. The second half of the encoding of a tile
  describes which tiles are allowed to occur above the tile.
  	The query $Q_1$ is
	\begin{equation*}
	\begin{split}
	\widehat {i_1}\cdots \widehat {i_n} \left(\sum_{i=0}^{n-2} \sum_{(v_1,v_2)\in \tilde H} \! (\widehat T_1)^i  \widehat {v_1}\widehat {v_2} (\widehat T_1)^{n-i-2} \right)^*\! \widehat{f_1}\cdots \widehat {f_n}\; . 
	 \end{split}
 \end{equation*}
 We note that $Q_1$ encodes exactly the tilings without horizontal errors, due
 to the imposed restrictions.
 
 The query $Q_2$ is $\tilen (\tilen + \tiley +\tilen + \tileY)^{(2n-1)|T|-1} \tileY$ and matches
 exactly those positions where a vertical error occurs, exploiting the encoding
 of vertical constraints in the second half of each tile's encoding.
 \end{proof}

\makeatletter{}\section{Simple Transitive Closures}

In this section, we investigate what happens if we consider fragments 
that only allow singleton transitive closures, that is, transitive closures of
single symbols. Our first results imply a number of \pitwo-results in Table~\ref{tab:complexity-smaller}.

\begin{theorem} \label{theo:lowerbound-aastar-in-a}
	$\querycon$ of $\crpq(a,a^*)$ in $\crpq(a)$ is \pitwo-hard, even if the size of the alphabet is fixed.
\end{theorem}
\begin{proof}[Proof sketch]
	We use a similar reduction as in Theorem~\ref{theo:A-in-a}. The only
 change we make is that we replace the expressions $t+f$ in $Q_1$ with $t^*
 f$-paths. Intuitively, $Q_1$ sets a variable $x_i$ to true if and only if
 there exists at least one $t$-edge after the $x_i$-edge. The query $Q_2$ is
 not changed.
\end{proof}

\begin{corollary}\label{cor:A-astar-in-crpq}
  Containment of $\crpq(A,a^*)$ in \crpq is in \pitwo.
\end{corollary}
\begin{proof}
  This will be a corollary of Theorem~\ref{theo:w-in-crpq}, since 
  $\crpq(A,a^*)$ is a fragment of $\crpq(W)$.
\end{proof}

On the other hand, if we allow arbitrary queries on the left and simple transitive closure on the right-hand query, the problem becomes \pspace-hard.
\begin{theorem}\label{theo:crpq-in-a-astar}
	$\querycon$ of $\crpq$ in $\crpq(a,a^*)$  is \pspace-complete, even if the size of the alphabet is fixed. 
\end{theorem}
\begin{proof}[Proof sketch]
	We adapt the encoding in the proof of Theorem~\ref{theo:crpq-in-A}, by (a) replacing each symbol 
	$\sigma \in  \{\tileN,\tileY,\tilen,\tiley\}$ with $\sigma\$$, where $\$$ is a new symbol, 
	and (b) replacing $Q_2$ with $\tilen\$(\tileN^*\tileY^*\tilen^*\tiley^*\$)^{(2K-1)|T|-1} \tileY\$$.
\end{proof}

Interestingly, the complexity of containment can drop by adding distinguished
variables to the query:
\begin{propositionrep}
	The complexity of $\querycon$ of (1) $\crpq$ in $\crpq(A)$ and  (2) $\crpq$ in
  $\crpq(a,a^*)$ is in $\pitwo$ if every component of each query contains at least one
  distinguished variable. 
\end{propositionrep}
\begin{proof}
	The main reason for this drop of complexity is that the queries on the right side allow only very restricted navigation. Therefore, each component has to embed ``close'' to its distinguished node(s). \tina{which allows the use of polynomial-sized canonical models.}
	 Due to the restricted language of $Q_2 \in \crpq(A)$ in case (1), 
	 	 the components are mapped to nodes which are reachable 
	  by a path of length $\leq d$ from the distinguished nodes, where $d$ is some polynomial in the size of $Q_2$. 
	    In case (2), the argumentation is more complex, as the
	    query $Q_2 \in \crpq(a,a^*)$ can have arbitrarily long paths due to the $a^*$, but we can compress the paths we need to consider by (i) limiting the length of paths using the same symbol and (ii) limiting the number of symbol changes. 
	    Thus we only need to consider paths of polynomial length around distinguished nodes.
	    \tina{maybe add:}
	    Limiting the length of each $a$-path is motivated by the standard argument that if we want to test $Q_1 \subseteq Q_2$ for the fragment $(a,a^*)$, we only need to replace each transitive edge in $Q_1$ by at most $|Q_2|+1$ many normal edges. Restricting the number of symbol changes is immediate from the fragment $(a,a^*)$: As each edge can only overcome one sort of symbol, each change requires a new edge, thus the number of symbol changes is limited to $|Q_2|$.
\end{proof}

Finally we show that, as long as the right query only has single symbols under
Kleene closures, query containment remains \pspace-complete.
\begin{theorem}\label{theo:crpq-in-A-astar}
	\querycon of \crpq in $\crpq(A,a^*)$ is \pspace-complete.
\end{theorem}
\begin{proof}
	The lower bound is immediate from Theorem~\ref{theo:crpq-in-A}. For the upper
  bound we provide a \pspace-algorithm for non-containment. Let $Q_1 \in \crpq$,
  $Q_2 \in \crpq(A,a^*)$, and $\#$ be a symbol not appearing in $Q_1$ and $Q_2$.
  We first note that each component of $Q_2$ can express at most $|Q_2|$ many
  label changes on a path. Hence it suffices if the algorithm stores just the
  part of a path that corresponds to the last $|Q_2|$ label changes.
  Furthermore, a standard pumping argument yields that, in a counterexample, the
  length of segments that only use a single label can be limited to $|Q_1|+|Q_2|$.

  Therefore, for each atom of $A = xRy$ of $Q_1$, the \pspace-algorithm guesses
  words $u_A, v_A$ of length at most $|Q_2|\times(|Q_1|+|Q_2|)$, such that $u_A
  \Sigma^* v_A \cap L(R) \neq \emptyset$ and, if $u_A$ or $ v_A$ has less than
  $|Q_2|$ many label changes, then $u_av_a \in L$. We guess a component of
  $Q'_2$ and check that
  \begin{enumerate}[(1)]
	\item $Q'_2$ cannot be embedded in $Q'_1$, where $Q'_1$ is the KB resulting
    from replacing each atom $A= x R y$ with the path $u_A \cdot s_\# \cdot
    v_A$, where $s_\# = \varepsilon$ if $u_A$ or $v_A$ contains less than
    $|Q_2|$ label changes and $s_\# = \#$ otherwise; and
	\item for every atom $A = x R y$ of $Q_1$ such that $u_a$ and $v_a$ have $Q_2$
    many label changes there is $w \in u_A \Sigma^* v_A \cap L(R)$ such that
    $Q'_2$ cannot be embedded in $w$.
  \end{enumerate}
  If tests (1) and (2) succeed, we found a knowledge base into which $Q_1$ can
  be embedded, but $Q_2$ cannot. Test (1) is in \conp as $Q'_1$ has size
  polynomial in $Q_1$ and $Q_2$. Test (2) is in polynomial space, as the
  restricted language of $Q_2$ allows us to guess and verify the existence of
  $w$ on the fly while only keeping the path corresponding to the last $|Q_2|$
  label changes in memory with length at most $|Q_2|\times(|Q_1|+|Q_2|)$.
\end{proof}

\makeatletter{}\section{Transitive Closures of Sets}

In this section we show that adding just a little more expressiveness makes
containment \expspace-complete. This high complexity may be surprising,
considering that it already holds for $\crpq(a,A^*)$ queries, which is a fragment that
merely extends ordinary conjunctive queries \emph{by adding transitive reflexive
  closures of simple disjunctions}. Our proof is inspired on the hardness proof in
\citep{CalvaneseGLV-kr00} for general CRPQs, but we need to add a number of
non-trivial new ideas to make it work for $\crpq(a,A^*)$.
\paragraph{Disjunction creation.}
\label{para:dis-creat}
A significant restriction that is imposed on $\crpq(a,A^*)$ is that the non-transitive atoms are not
allowed to have disjunctions in their expressions. We get around this
by the following idea that generates disjunctive \emph{bad patterns} out of
conjunctions --- we use a similar idea in our next proof.

Consider the following query $Q_2$ where $\ell$ is a special helper symbol, 
$y_1 ~ \ell^* \cdot s_1 \cdot \ell \cdot s_2 \cdot \ell^*y_2$.
For query $Q_1$ given by 
 $\bigwedge_{\sigma \in \Sigma{\setminus}\{\ell\}}x_1 \sigma x_1 \wedge  x_1 \ell ~(\Sigma\setminus\{\ell\})^* ~\ell x_2
 \wedge \bigwedge_{\sigma \in \Sigma\setminus\{\ell\}}x_2 \sigma x_2$
 it is clear that $Q_1$ allows for exactly two $\ell$, and hence, if $Q_1$ would
 be contained in $Q_2$, one of the patterns $s_1$ or $s_2$ has to be be matched
 to the $(\Sigma\setminus\{\ell\})^*$ fragment in the middle. Essentially, we
 capture all bad patterns matching either $s_1$ or $s_2$, thereby ``creating''
 the result of a disjunction.

\begin{theoremrep}\label{theo:a-Astar}
	$\querycon$ of $\crpq(a,A^*)$ in $\crpq(a,A^*)$ is \expspace-hard, even if the
  size of the alphabet is fixed. 
\end{theoremrep}
\begin{proofsketch}
We reduce from the exponential width corridor tiling problem. That is, we have
	\begin{itemize}
		\item a finite set $T=\{t_1, \ldots, t_m\}$ of tiles,
		\item initial and final tiles $t_I, t_F \in T$, respectively,
		\item horizontal and vertical constraints $H,V \subseteq T \times T$,
		\item a number $n\in \nat$ (in unary),
	\end{itemize}
  and we want to check if there is a $k \in \nat$ and a tiling function $\tau\colon
  \{1,\ldots,k\}\times \{1,\ldots, 2^n\}\to T$ such that $\tau(1,1) = t_I$,
  $\tau(k,2^n)=t_F$, and all horizontal and vertical constraints are satisfied.
  In order to have a fixed alphabet, we encode tiles from $T$ as words from
  $\{\tileN, \tileY\}^m$. The $i$-th tile $t_i$ is encoded as $\hat t_i=
  \tileN^{i-1}\tileY \tileN^{m-i-1} \in \{\tileN, \tileY\}^m$.
  
  A tiling $\tau$ is encoded as a string over the alphabet
  $\mathbb{B}=\{\$,0,1,\tileN,\tileY,\#\}$, where $\$$ is the row separator, $0$
  and $1$ are used to encode addresses for each row of the tiling from $0$ to
  $2^n-1$ as binary numbers, $\#$ separates the individual bits of an address,
  and $\tileN$ and $\tileY$ are used to encode the individual tiles. We
  visualize a tiling as a matrix with $k$ rows of $2^n$ tiles each. An example
  of a tiling $\tau$ with $n=3$ is below:
  \begin{align*}
  		&\widehat{\tau(k,1)} 0\#0\#0 \widehat{\tau(k,2)} 0\#0\#1 &\cdots& &\widehat{\tau(k,2^3)} 1\#1\#1& &\$&\\
  		& \qquad\qquad\qquad\vdots&\reflectbox{$\ddots$}&  &\vdots\qquad& &\vdots&\\
		\$\,&\widehat{\tau(1,1)} 0\#0\#0 \widehat{\tau(1,2)} 0\#0\#1 &\cdots& &\widehat{\tau(1,2^3)} 1\#1\#1&& \$&
  \end{align*}

  The queries $Q_1$ and $Q_2$ use the alphabet $\mathbb{A}=\mathbb{B} \cup
  \{[,],\langle ,\rangle,b,\star\}$. This new set contains helper symbols
            $[$ and $]$ which we use for disjunction creation (in a
  similar way as we explained before the Theorem statement), and $\langle$ and
  $\rangle$ denote the start and end of the tiling. The $b$-symbol is used for a
  special edge that we use for checking vertical errors.
              Query $Q_1$ is given in
  Figure~\ref{fig:lower-A-Astar:q1} and query $Q_2$ is sketched in
  Figure~\ref{fig:lower-A-Astar:q2}. For convenience we use
  $\mathbb{B}_{\langle\rangle}$ to abbreviate $\mathbb{B} \cup \{\langle ,
  \rangle\}$, $\mathbb{B}_{[]\langle\rangle}$ to abbreviate
  $\mathbb{B}_{\langle\rangle} \cup \{[,]\}$, and $\mathbb{B}_{\overline \$}$ to
  abbreviate $\mathbb{B} \setminus \{\$\}$.
  
  The intuition is that the tiling is encoded in the $\mathbb{B}^*$-edge of
  $Q_1$, i.e. the only edge that is labeled by a language that is not a single
  symbol. The query $Q_2$ consists of a sequence  of \emph{bad patterns}, one for
  each possible kind of violation of the described encoding or the horizontal
  and vertical constraints. The queries are designed in such a way that $Q_2$
  cannot be embedded if a valid tiling is encoded in a canonical model of $Q_1$.
  Otherwise, at least one of the bad patterns can be embedded in the encoding of
  the tiling. The other bad patterns can be embedded at the nodes $z_2$ and $z_7$
  of $Q_1$, as these nodes have one self loop for every symbol of the alphabet
  except $\star$.

  We can easily design (sets of) patterns, where each pattern is a simple path, 
  to catch the following errors: malformed encoding of a tile, malformed
  encoding of an address, non-incrementing addresses, missing initial or final
  $\$$, wrong initial or final tile, and an error in the horizontal constraints.

  The most difficult condition to test is an error in the vertical constraints, which
  we encode with the pattern $G^{t,t'}$ for every $(t,t') \notin V$, given by
  \[\bigwedge_{1\leq i\leq n} G_i^{t,t'} \wedge \bigwedge_{\substack{i,j \in \{1,\dots,n\} \\ 
        c,d \in \{0,1\}; |i-j|=1}} (x_{i,c}^{t,t'}Lx_{j,d}^{t,t'} \wedge
    y_{i,c}^{t,t'}Ly_{j,d}^{t,t'})\;,\] where $G_i^{t,t'}$ is given in
  Figure~\ref{fig:lower-A-Astar:gi} and $L=b^* \mathbb{B}_{\langle\rangle}^*
  b^*$. We first explain the intuition behind $G_i^{t,t'}$.
  We assume that the vertical error occurs at tile $t$ having $0$ as $i$-th bit
  of its address. In that case, the variable $x_{i,0}^{t,t'}$ should be embedded
  just before the encoding of $t$, while $y_{i,0}^{t,t'}$ should be embedded in
  the next row just after the tile $t'$ with the same $i$-th bit. This is
  enforced as there is one $\$$ between $x_{i,0}^{t,t'}$ and $y_{i,0}^{t,t'}$,
  ensuring that both variables occur in consecutive rows.
  The variables $x_{i,1}^{t,t'}$ and $y_{i,1}^{t,t'}$ are simply
  embedded at the node corresponding to $z_6$ of $Q_1$.

  In the case that the $i$-th bit is 1, we embed $x_{i,0}^{t,t'}$ and
  $y_{i,0}^{t,t'}$ at $z_3$, while $x_{i,1}^{t,t'}$ and $y_{i,1}^{t,t'}$ are
  embedded at the tiles violating the vertical constraint, as described in the
  previous case.

	\begin{figure}[t]
		\centering
		\begin{tikzpicture}[auto, >=latex]
		\def\stretch{1.15};
		\def\ystretch{1};
		
		\node [] (v0) at (\stretch*-1,\ystretch) {};
		\node [] (v1) at (\stretch*0,\ystretch) {};
		\node [] (v2) at (\stretch*1,\ystretch) {};
		\node [] (v3) at (\stretch*2,\ystretch) {};
		\node [] (v4) at (\stretch*3,\ystretch) {};
		\node [] (v5) at (\stretch*4,\ystretch) {};
		\node [] (v6) at (\stretch*5,\ystretch) {};
		\node [] (v7) at (\stretch*6,\ystretch) {};

    \node[yshift=2.5mm] at (v0) {$z_1$};
    \node[yshift=2.5mm] at (v1) {$z_2$};
    \node[yshift=2.5mm] at (v2) {$z_3$};
    \node[yshift=2.5mm] at (v3) {$z_4$};
    \node[yshift=2.5mm] at (v4) {$z_5$};
    \node[yshift=2.5mm] at (v5) {$z_6$};
    \node[yshift=2.5mm] at (v6) {$z_7$};
    \node[yshift=2.5mm] at (v7) {$z_8$};
		
		\foreach \i in {0,1,2,3,4,5,6,7}{
			\fill (v\i) circle (2pt);
		}
		\path[->]
		(v0) edge	[] node[below]  {$\star$} (v1)
		(v1) edge	[loop below] node  {$\mathbb{B}_{[]\langle\rangle}$} (v1)
		(v1) edge	[loop below, min distance=6.5mm, out = -60, in = -120,looseness=10] node  {} (v1)
		(v1) edge	[] node[below]  {$[$} (v2)
		(v2) edge	[loop below]  node  {$\mathbb{B}_{\langle\rangle}$} (v2)
		(v2) edge	[loop below, min distance=6.5mm, out = -60, in = -120,looseness=10]  node  {} (v2)
		(v2) edge	[] node[below]  {$\langle $} (v3)
		(v3) edge	[] node[below]  {$\mathbb{B}^*$} (v4)
		(v4) edge	[] node[below]  {$\rangle$} (v5)
		(v5) edge	[loop below] node  {$\mathbb{B}_{\langle\rangle}$} (v5)
		(v5) edge	[loop below, min distance=6.5mm, out = -60, in = -120,looseness=10] node  {} (v5)
		(v5) edge	[] node[below]  {$]$} (v6)
		(v6) edge	[loop below]  node  {$\mathbb{B}_{[]\langle\rangle}$} (v6)
		(v6) edge	[loop below, min distance=6.5mm, out = -60, in = -120,looseness=10]  node  {} (v6)
		(v6) edge	[] node[below]  {$\star$} (v7)
		
		(v5) edge	[bend right = 30,swap] node  {$b$} (v2)
		;
		\end{tikzpicture}
		\caption{Query $Q_1$ in the proof of Theorem~\ref{theo:a-Astar}.
      Double-self-loops indicate a distinct self-loop for every single symbol,
      i.e., not a self-loop labeled with the alphabet. }
		\label{fig:lower-A-Astar:q1}
	\end{figure}
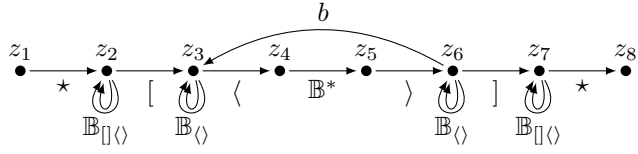
		\begin{figure}[t]
		\centering
		\begin{tikzpicture}[auto, >=latex]
		\def\stretch{.8};
		\def\ystretch{.8};
		
		\node [] (v0) at (\stretch*-.7,\ystretch) {};
		\node [] (v1) at (\stretch*0,\ystretch) {};
		\node [] (v2) at (\stretch*1,\ystretch) {};
		\node [] (v3) at (\stretch*2,\ystretch) {};
		\node [] (v4) at (\stretch*3,\ystretch) {};
		\node [] (v5) at (\stretch*4,\ystretch) {};
		\node [] (v6) at (\stretch*5,\ystretch) {};
		\node [] (v7) at (\stretch*6,\ystretch) {};
		\node [] (v8) at (\stretch*7,\ystretch) {};
		\node [] (v9) at (\stretch*7.7,\ystretch) {};
		
		\foreach \i in {0,1,2,3,4,5,6,7,8,9}{
			\fill (v\i) circle (2pt);
		}
		\path[->]
		(v0) edge	[] node  {$\star$} (v1)
		(v1) edge	[] node  {$[$} (v2)
				(v3) edge	[] node  {$]$} (v4)
				(v5) edge	[] node  {$[$} (v6)
				(v7) edge	[] node  {$]$} (v8)
		(v8) edge	[] node  {$\star$} (v9)
		;
		
		\node [] (dots) at (\stretch*3.5,\ystretch) {$\cdots$};
		\node [] (b1) at (\stretch*1.5,\ystretch) {$B_1$};
		\node [] (b2) at (\stretch*5.5,\ystretch) {$B_\ell$};
				\draw
		(\stretch*1,\ystretch*.5)--(\stretch*2,\ystretch*.5)--(\stretch*2,\ystretch*1.5)--(\stretch*1,\ystretch*1.5)-- (\stretch*1,\ystretch*.5)
		;
		\draw
		(\stretch*5,\ystretch*.5)--(\stretch*6,\ystretch*.5)--(\stretch*6,\ystretch*1.5)--(\stretch*5,\ystretch*1.5)-- (\stretch*5,\ystretch*.5)
		;
		\end{tikzpicture}
		\caption{Query $Q_2$ in the proof of Theorem~\ref{theo:a-Astar}. The $B_i$ denote ``bad patterns'' described in the proof;  each $B_i$ has a `left' and `right' distinguished variable as in the picture.}
		\label{fig:lower-A-Astar:q2}
	\end{figure}
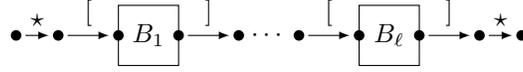

  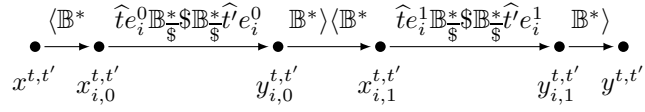
\begin{figure}[ht]
		\centering
		\begin{tikzpicture}[auto, >=latex]
		\def\stretch{3.4};
		\def\ystretch{1};
		
		\node [label=below:{$x^{t,t'}$}] (v1) at (\stretch*0.65,\ystretch) {};
		\node [label=below:{$x_{i,0}^{t,t'}$}] (v2) at (\stretch*.9,\ystretch) {};
		\node [label=below:{$y_{i,0}^{t,t'}$}] (v3) at (\stretch*1.6,\ystretch) {};
		\node [label=below:{$x_{i,1}^{t,t'}$}] (v4) at (\stretch*2,\ystretch) {};
		\node [label=below:{$y_{i,1}^{t,t'}$}] (v5) at (\stretch*2.7,\ystretch) {};
		\node [label=below:{$y^{t,t'}$}] (v6) at (\stretch*2.95,\ystretch) {};
		
		\foreach \i in {1,2,3,4,5,6}{
			\fill (v\i) circle (2pt);
		}
		\path[->]
		(v1) edge	[] node [yshift=1pt] {$	\langle \mathbb{B}^*$} (v2)
		(v2) edge	[] node [yshift=-1.9pt] {$\widehat{t} e_i^0 \mathbb{B}_{\overline \$}^*\$\mathbb{B}_{\overline \$}^* \widehat{t'} e_i^0$} (v3)
		(v3) edge	[] node [yshift=1pt] {$\mathbb{B}^*\rangle\langle \mathbb{B}^*$} (v4)
		(v4) edge	[] node [yshift=-1.9pt] {$\widehat{t} e_i^1 \mathbb{B}_{\overline \$}^*\$\mathbb{B}_{\overline \$}^* \widehat{t'} e_i^1$} (v5)
		(v5) edge	[] node [yshift=1pt] {$\mathbb{B}^*\rangle$} (v6)
		;

		\end{tikzpicture}
		\caption{Subquery $G_i^{t,t'}$ in the proof of
      Theorem~\ref{theo:a-Astar}. Here, $e_i^a=(\{0,1\}^*\#)^{i-1} a (\#
      \{0,1\}^*)^{n-i-1}$ is the language enforcing the $i$-th bit to be $a$.}
		\label{fig:lower-A-Astar:gi}
	\end{figure}

  Altogether, $G_i^{t,t'}$ verifies that there are positions $v$ and $w$ in
  consecutive rows of the encoding such that the tiles adjacent to $v$ and $w$
  would violate the vertical constraints and the positions agree on the $i$-th
  bit of the address. To ensure that the positions $v$ and $w$ agree on
  \textit{all} $n$ bits of the address we have to ensure that the $n$ patterns
  $G_1^{t,t'}, \dots, G_n^{t,t'}$ all refer to the \textit{same} two positions
  in the tiling. This is why we have the additional conjuncts with language $L
  =b^* \mathbb{B}_{\langle\rangle}^* b^*$ in $G^{t,t'}$. The language $L$ is
  chosen to ensure that there exists exactly one node $v$ in the tiling such
  that all the variables $x_{1,j}^{t,t'}, \dots, x_{n,j}^{t,t'}$, for $j \in
  \{0,1\}$ are either embedded at $v$, at the node corresponding to $z_3$ from
  $Q_1$, or at the node corresponding to $z_6$ from $Q_1$. If there were two
  variables $x_{i,c}^{t,t'}$ and $x_{j,d}^{t,t'}$ embedded at different
  positions between $z_3$ and $z_6$ then there is a $k$ and
  $\tilde{c},\tilde{d}$ such that $x_{k,\tilde{c}}^{t,t'}$ and
  $x_{k+1,\tilde{d}}^{t,t'}$ are embedded at different positions and thus at
  least one of the conjuncts $x_{k,\tilde{c}}^{t,t'}Lx_{k+1,\tilde{d}}^{t,t'}$
  and $x_{k+1,\tilde{d}}^{t,t'}Lx_{k,\tilde{c}}^{t,t'}$ has to be violated, as
  the symbol $b$ can be read only at the beginning or end of a string in $L$
  (recall that $b \notin \mathbb{B}_{\langle\rangle}$). The argument for the
  $y$-variables and the position $w$ is analogous.

	To conclude, whenever there exists a valid tiling, we have a canonical
  knowledge base with the encoding of a tiling occurring between $z_4$ and $z_5$. To embed $Q_2$ 
  into this, we need to span the full length flanked by the $\star$'s in the start and the end. 
  Thanks to (i) the symbols $[, ]$ flanking the bad patterns $B_i$ in $Q_2$, and (ii)
  the presence of these symbols only at edges from nodes $z_2, z_7$ in $Q_1$, 
    at least one of the bad patterns must embed into the 
  part between $z_3$ and $z_6$.  
  If there is no error,  we cannot embed $Q_2$, and hence no $B_i$ can be
  mapped between $z_3$ and $z_6$ and we have $Q_1 \not\subseteq Q_2$. On the
  other hand, when there is no valid tiling, for each canonical knowledge base with a
  `guessed' tiling, $Q_2$ maps one of the $B_i$ between $z_3$ and $z_6$, and can hence 
  embed completely from $\star$ to $\star$,  giving $Q_1 \subseteq Q_2$. 
\end{proofsketch}  
\begin{proof}
	We reduce from the exponential width corridor tiling problem. That is, we have
	\begin{itemize}
		\item a finite set $T=\{t_1, \ldots, t_m\}$ of tiles
		\item some initial and final tiles $t_I, t_F \in T$
		\item horizontal and vertical constraints $H,V \subseteq T \times T$
		\item a number $n\in \nat$ (in unary)
	\end{itemize}
   and we want to know if there is $k \in \nat$ and a tiling function $\tau\colon \{1,\ldots,k\}\times \{1,\ldots, 2^n\}\to T$ so that $\tau(1,1) = t_I$, $\tau(k,2^n)=t_F$, and,  all horizontal and vertical constraints are satisfied. 
   In order to have a fixed alphabet, we encode tiles from $T$ as words from $\{\tileN, \tileY\}^m$.    The tile $t_i$ is encoded as $\hat t_i= \underbrace{\tileN \dotsb \tileN}_{i-1} \tileY \underbrace{\tileN \dotsb \tileN}_{m-i-1} \in \{\tileN, \tileY\}^m$. (Remember, $m = |T|$.)
   
   A tiling $\tau\colon \{1,\ldots,k\}\times \{1,\ldots, 2^n\}\to T$ will be encoded as a string from $\mathbb{B}^*$ for $\mathbb{B}=\{\$,0,1,\tileN,\tileY,\#\}$. We visualize a tiling as a matrix of some $k \in \mathbb{N}$ number of rows each with $2^n$ tiles. Tiles in each row are addressed using an $n$-bit address, the bits being separated by a $\#$ symbol. For $n=3$, a possible encoding is as follows.\\
	\begin{align*}
		\$\,&\widehat{\tau(1,1)} 0\#0\#0 \widehat{\tau(1,2)} 0\#0\#1 &\cdots& &\widehat{\tau(1,2^3)} 1\#1\#1&& \$&\\
		& \qquad\qquad\qquad\vdots&\ddots&  &\vdots\qquad& &\vdots&\\
		&\widehat{\tau(k,1)} 0\#0\#0 \widehat{\tau(k,2)} 0\#0\#1 &\cdots& &\widehat{\tau(k,2^3)} 1\#1\#1& &\$&
	\end{align*}
	
	We construct the queries $Q_1, Q_2$ over the alphabet $\mathbb{A}=\mathbb{B} \cup \{[,],\langle ,\rangle,b,\star\}$. Intuitively, $Q_1$ will be used to guess a particular tiling (represented as a string on $\mathbb{B}$). $Q_2$ will be used to identify errors in the tiling. The symbols in $\{[,],\langle ,\rangle,b,\star\}$ are not part of the encoding but are used as helper symbols to capture bad patterns. Essentially then, $Q_1\subseteq Q_2$ whenever any tiling (canonical knowledge base) guessed by $Q_1$ contains an error and hence $Q_2$ can be embedded in it. On the other hand, $Q_1 \not\subseteq Q_2$ implies that there exists a tiling which is error free.
	
	Query $Q_1$ is described in Figure~\ref{fig:lower-A-Astar:q1} and $Q_2$ in Figure~\ref{fig:lower-A-Astar:q2}. 
	To complete the construction of $Q_2$ we  need to describe the `bad patterns', which we will do next. Every bad pattern described below, except the last one (which captures errors in vertical constraints),  is a directed path from the leftmost variable to the rightmost variable of the pattern. We now describe the labels of the paths, and we finally describe the last crucial bad pattern, capturing a mismatch of the vertical constraints.
	\begin{itemize}
		\item Bad encoding. We have a bad path with label $\langle \mathbb{B}^* w\mathbb{B}^* \rangle$ for every $w$ as described below.
			\begin{itemize}			
				\item The number of occurrences of $\tileY$ and the length of a tile encoding:
				\begin{itemize}				 
					\item more than one $\tileY$: $w = \tileY \tileN^i \tileY$ for every $i \leq m-1$.
					\item no $\tileY$: $w = \tileN^m$.
				\end{itemize}
				\item Bad encoding of addresses
				\begin{itemize}
					\item no alternation between $\{0,1\}$ and $\#$: $w= 00$, $01$, $10$,$11$, $\#\#$
					\item begins/ends with $\#$: $w = \#a$, $a\#$ for $a \in \mathbb{B}\setminus\{0,1\}$
					\item bad length of bit string
				\begin{itemize}
					\item too many: $w=(\#\{0,1\}^*)^n$
					\item too few: $w=a_1 (\{0,1\}^* \#)^{n'} \{0,1\}^* a_2$ for every $n' <  n-1$ and $a_1, a_2 \in \mathbb{B} \setminus \{0,1,\#\}$
				\end{itemize}
			\end{itemize}
			
		\end{itemize}
		\item The start or end of the encoding has a problem. We have a bad pattern for each of the following languages.
		\begin{itemize}
		\item Encoding does not start and end with $\$ $.
		\begin{itemize}
      		\item $\langle a\mathbb{B}^*\rangle$ for every $a \in \mathbb{B}\setminus\{\$\}$
      		\item $\langle \mathbb{B}^*a\rangle$ for every $a \in \mathbb{B}\setminus\{\$\}$
		\end{itemize}
      	\item Initial or final tiles are absent. 
      	\begin{itemize}
      		\item $\langle  \$ \cdot \hat{t} \cdot \mathbb{B}^* \rangle$ for every $t \in T \setminus \set{t_I}$
      		\item $\langle \mathbb{B}^* \cdot \hat{t}\cdot  \$\rangle$ for every $t \in T \setminus \set{t_F}$
      	\end{itemize}
		\end{itemize}

        \item There is  some row, for which the bit string representing the address  does not increment correctly from $0^n$ to $1^n$. We have a bad pattern with a path having label $\langle \mathbb{B}^* w \mathbb{B}^* \rangle$ for every expression $w$ as described below.
        \begin{itemize}
        	\item First address of a row is not $0^n$: $w =  \$ \{\tileN,\tileY\}^*\{0,\#\}^*1$
        	\item Last address of a row is not $1^n$: $w = 0 \{1,\#\}^*\$$
        	\item There are consecutive addresses in the same row so that the right-hand address is not the increment of the left-hand one: For every $i,j,k \leq n$, $a \in \{0,1\}$: 
        	\begin{itemize}
        	\item $i$-th bit does not change from $0$ to $1$ when it should: $w = 0(\#1)^{n-i} \cdot \{\tileN, \tileY\}^* \cdot (\{0,1\}^*\#)^{i-1} 0$ 
        	\item $j$-th bit does not change from $1$ to $0$ when it should ($j>i$): $w = 0(\#1)^{n-i} \cdot \{\tileN, \tileY\}^* \cdot (\{0,1\}^*\#)^{i} \set{0,1,\#}^*1$ 
        	\item $i$-th bit flips its value when it shouldn't: $w = (\{0,1\}^*\#)^{i-1} a\#(\{0,1\}^*\#)^j 0 (\#1)^k \cdot \{\tileN, \tileY\}^* \cdot (\{0,1\}^*\#)^{i-1} (1-a)$ 
        	        	        	\item after $1^n$ there is no $\$$: $w = 1(\#1)^{n-1} \cdot (\mathbb B \setminus \set{\$})$.             \end{itemize}
        \end{itemize}
        \item A pair of horizontally consecutive tiles does not satisfy the horizontal constraints: For every $(t,t') \in T^2 \setminus H$ we have a path bad pattern with language $\hat{t} \cdot \{0,1,\#\}^* \cdot \hat{t'}$.
        \item A pair of vertically consecutive tiles does not satisfy the vertical constraints. This is the most challenging condition to test, for which we need to use a more complex query than just a path query. For every $(t,t') \in T^2 \setminus V$ we construct the bad pattern \[\bigwedge_{1\leq i\leq n} G_i^{t,t'} \wedge \bigwedge_{1\leq i\leq n-1} P_i^{t,t'}\] where $G_i^{t,t'}$ and $P_i^{t,t'}$ are defined in Figures~\ref{fig:lower-A-Astar:gi} and \ref{fig:lower-A-Astar:pi} respectively. The left and right variables of the bad pattern are $x^{t,t'}$ and $y^{t,t'}$ respectively.
	\end{itemize}
	
	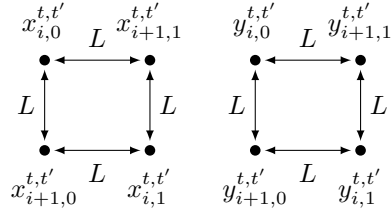
\begin{figure}[ht!]
		\centering
		\begin{tikzpicture}[auto, >=latex]
		\def\stretch{1.4};
		\def\ystretch{.3};
				\node [label=above:{$x_{i,0}^{t,t'}$}] (v1) at (\stretch*1,\ystretch) {};
		\node [label=below:{$x_{i,1}^{t,t'}$}] (v2) at (\stretch*2,\ystretch*-3) {};
		\node [label=below:{$x_{i+1,0}^{t,t'}$}] (v3) at (\stretch*1,\ystretch*-3) {};
		\node [label=above:{$x_{i+1,1}^{t,t'}$}] (v4) at (\stretch*2,\ystretch) {};

		\node [label=above:{$y_{i,0}^{t,t'}$}] (v5) at (\stretch*3,\ystretch) {};
		\node [label=below:{$y_{i,1}^{t,t'}$}] (v6) at (\stretch*4,\ystretch*-3) {};
		\node [label=below:{$y_{i+1,0}^{t,t'}$}] (v7) at (\stretch*3,\ystretch*-3) {};
		\node [label=above:{$y_{i+1,1}^{t,t'}$}] (v8) at (\stretch*4,\ystretch) {};

		\foreach \i in {1,2,3,4,5,6,7,8}{
			\fill (v\i) circle (2pt);
		}
		\path[<->]
				(v1) edge	[swap] node  {$L$} (v3)
		(v1) edge	[] node  [yshift=1.5pt]  {$L$} (v4)
		(v2) edge	[] node  [yshift=-1.5pt] {$L$} (v3)
		(v2) edge	[swap] node  {$L$} (v4)

				(v5) edge	[swap] node  {$L$} (v7)
		(v5) edge	[] node  [yshift=1.5pt]  {$L$} (v8)
		(v6) edge	[] node  [yshift=-1.5pt] {$L$} (v7)
		(v6) edge	[swap] node  {$L$} (v8)
				;

		\end{tikzpicture}
		\caption{Subquery $P_i^{t,t'}$ in the proof of Theorem~\ref{theo:a-Astar}. Observe that $P_i^{t,t'}$ and $P_i^{t,t'}$ 
		share variables. 		
				Here, $L = b^* (\mathbb{B}^* \cup \{\langle ,\rangle\})^* b^*$. Double-headed arrows such as $x \xleftrightarrow L y $ mean that there are edges $x \xrightarrow L y$ and $x \xleftarrow L y$.
				}
		\label{fig:lower-A-Astar:pi}
	\end{figure}
	
	We now show correctness. 
	If there is a correct tiling, its encoding can be constructed by $Q_1$, more precisely, between nodes $z_4$ and $z_5$ (called the middle-path or \textit{MP} in short) annotated by $\mathbb{B}^*$. We show that no ``bad pattern'' of $Q_2$ can match into it. Since it is formatted and encoded correctly, we can exclude those bad pattern. Since it is easy to see that the initial and final tile are correct and horizontal constraints must be satisfied, we will focus on vertical constraints only. Let us therefore assume towards contradiction that there is $(t,t') \notin V$ such that 
	$$\bigwedge_{1\leq i\leq n} G_i^{t,t'} \wedge \bigwedge_{1\leq i\leq n-1} P_i^{t,t'}$$
	can be mapped into the \textit{middle-path} of $Q_1$. 
	Let the $x$'s (short for $x^{t,t'}_{i,0}, x^{t,t'}_{i,1},x^{t,t'}_{i+1,0},x^{t,t'}_{i+1,1}$)  and $y$'s be mapped such that all $G_i^{t,t'}$ hold. One of the $x$ (and $y$) must be mapped on node $z_3$ (called the 2nd double-self-loop or \textit{2DSL} in short) or node $z_6$ (called 3rd double-self-loop or \textit{3DSL} in short) in $Q_1$, while the other is mapped between $z_4$ and $z_5$ (\textit{MP}). This is the case since no $\langle\rangle$ is permitted between $z_4$ and $z_5$ and hence consequently, we can have only one pair of $\langle\rangle$ on \textit{MP}. The second pair from $G_i^{t,t'}$ must hence be absorbed by one of \textit{2DSL} or \textit{3DSL}.
	We show that there is a $i$ such that $P_i^{t,t'}$ doesn't hold.
	Since the \textit{MP} of $Q_1$ encodes a correct tiling, there is no $j,k$ with $\tau(j,k)= t$, $\tau(j+1,k)=t'$. Therefore, not all $x$'s (excluding those on the \textit{DSL}s) are mapped onto the same node in the \textit{MP}, or not all $y$'s are. If the contrary were true it would mean that the addresses of $t$ and $t'$ match, $t$ and $t'$ would be vertically adjacent, and consequently the vertical relation would be violated.
		Then there is $z\in \{x,y\}$, $i \in \{1,\ldots, n-1\}$ and $a,c\in \{0,1\}$ such that $z_{i,a}^{t,t'}$ and $z_{i+1,c}^{t,t'}$ are mapped onto different nodes in the middle-path. Since the edges  in $P_i^{t,t'}$ are symmetric, we can assume w.l.o.g.\ that $z_{i,a}^{t,t'}$ is mapped to a node before $z_{i+1,c}^{t,t'}$ in $Q_1$. Then there is no $L$-path from $z_{i+1,c}^{t,t'}$ to $z_{i,a}^{t,t'}$, thus $P_i^{t,t'}$ doesn't hold.

	So let us assume that every tiling has an error (at least one). We will show that $Q_2$ has a ``bad pattern'' that can match into the middle part of $Q_1$. First of all, $Q_2$ has bad patterns accepting every string that is no correct encoding of a possible tiling (encoding of tiles or addresses are wrong, addresses are too short/long/don't increase correctly), so we only need to consider encodings of possible tilings to decide containment. 
	As all possible tilings have errors, we have to deal with: wrong initial or final tile, horizontal and vertical errors. The first few are easy to handle. We will how to handle vertical errors in more detail.
	Let us assume that the \textit{MP} of $Q_1$ represents a tiling with vertical error. We will show that there is a homomorphism mapping the pattern recognizing the vertical errors into this part. 
	Let $\tau(j,k) = t$ and $\tau(j+1,k) = t'$ with $(t,t')\notin V$ and $bin(k)= a_1 \cdots a_n$ be the binary representation of $k$. We can map $x^{t,t'}_{i,a_i}$ onto the node just before $\tau(j,k)$ and $y^{t,t'}_{i,a_i}$ onto the node just after $\tau(j+1,k) e_i^{a_i}$ for all $i$. The unused $x_{i,0}^{t,t'}$ and $y_{i,0}^{t,t'}$ can be mapped into the \textit{2DSL}, and the unused $x_{i,1}^{t,t'}$ and $y_{i,1}^{t,t'}$ into the \textit{3DSL} of $Q_1$.
	It remains to show that the $x$'s (short for $x^{t,t'}_{i,0}, x^{t,t'}_{i,1},x^{t,t'}_{i+1,0},x^{t,t'}_{i+1,1}$) are pairwise connected by $L$-paths, thus satisfy $P_i^{t,t'}$. The proof for the $y$'s is analogous. By our mapping, each $x$ is mapped on one of 3 nodes: the one right before $\tau(j,k)$, the \textit{2DSL} (node $z_3$) or the \textit{3DSL} (node $z_6$). (Similarly, each $y$ can be mapped on one of 3 nodes: the one just after $\tau(j+1,k)$, the \textit{2DSL} or the \textit{3DSL}). As all of these 3 nodes are pairwise connected by $L$-paths, and also the same node is connected to itself since $\varepsilon \in L$, the $x$ satisfy $P_i^{t,t'}$.
	
\paragraph{Bounding the treewidth of queries}	
	
We now observe the queries $Q_1$ and $Q_2$ considered above have bounded treewidth. This is trivially true for $Q_1$. We note that $Q_2$ is a linear composition of bad patterns or blocks $B_i$. If we can show that each such $B_i$ has a bounded treewidth, then we can also infer that $Q_2$ has a bounded treewidth. 
The blocks capturing bad encodings and horizontal constraints are composed of a single atom and hence clearly have bounded treewidth.
We thus focus on showing that the blocks $B_{t, t'}$ representing errors w.r.t. vertical constraints 
for each pair $t, t'$ of tiles have bounded treewidth.  We show that the treewidth for such a block $B^{t, t'}$ is in fact at most 9. Consider such a block  $B^{t,t'} = {\color{cyan} \bigwedge_{1\leq i\leq n} G_i^{t,t'}} \wedge {\color{red} \bigwedge_{1\leq i\leq n-1} P_i^{t,t'}}$. This block consists of the nodes $x^{t,t'}$, $y^{t,t'}$, $x_{i,0}^{t,t'}$ $x_{i,1}^{t,t'}$, $y_{i,0}^{t,t'}$ and $y_{i,1}^{t,t'}$ for $i\in [n]$. The structure of this block is as follows where the superscript $(t,t')$ is removed for clarity. Edges generated due to ${\color{cyan} \bigwedge_{1\leq i\leq n} G_i^{t,t'}}$ and ${\color{red} \bigwedge_{1\leq i\leq n-1} P_i^{t,t'}}$ are color coded accordingly. 
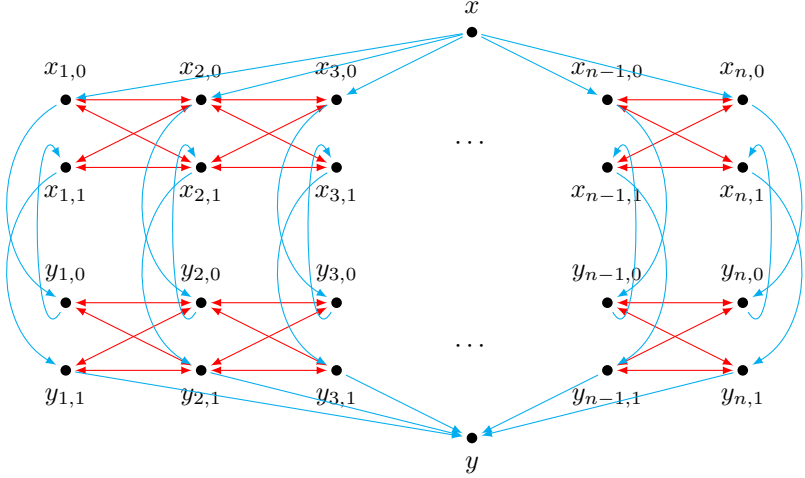
\begin{figure}[ht!]
		\centering
		\begin{tikzpicture}[auto, >=latex]
		\def\stretch{1.8};
		\def\ystretch{.3};
				\node [label=above:{$x$}] (c1) at (\stretch*1,\ystretch*0) {};
		\node [label=below:{$y$}] (c2) at (\stretch*1,\ystretch*-18) {};		
		
		\node [label=above:{$x_{1,0}$}] (v1) at (\stretch*-2,\ystretch*-3) {};
		\node [label=below:{$x_{1,1}$}] (v2) at (\stretch*-2,\ystretch*-6) {};
		\node [label=above:{$x_{2,0}$}] (v3) at (\stretch*-1,\ystretch*-3) {};
		\node [label=below:{$x_{2,1}$}] (v4) at (\stretch*-1,\ystretch*-6) {};		
		\node [label=above:{$x_{3,0}$}] (v5) at (\stretch*0,\ystretch*-3) {};
		\node [label=below:{$x_{3,1}$}] (v6) at (\stretch*0,\ystretch*-6) {};		

		\node [label={$\cdots$}] (c3) at (\stretch*1,\ystretch*-6) {};
		
		\node [label=above:{$x_{n-1,0}$}] (v7) at (\stretch*2,\ystretch*-3) {};
		\node [label=below:{$x_{n-1,1}$}] (v8) at (\stretch*2,\ystretch*-6) {};
		\node [label=above:{$x_{n,0}$}] (v9) at (\stretch*3,\ystretch*-3) {};
		\node [label=below:{$x_{n,1}$}] (v10) at (\stretch*3,\ystretch*-6) {};

		\node [label=above:{$y_{1,0}$}] (v11) at (\stretch*-2,\ystretch*-12) {};
		\node [label=below:{$y_{1,1}$}] (v12) at (\stretch*-2,\ystretch*-15) {};
		\node [label=above:{$y_{2,0}$}] (v13) at (\stretch*-1,\ystretch*-12) {};
		\node [label=below:{$y_{2,1}$}] (v14) at (\stretch*-1,\ystretch*-15) {};		
		\node [label=above:{$y_{3,0}$}] (v15) at (\stretch*0,\ystretch*-12) {};
		\node [label=below:{$y_{3,1}$}] (v16) at (\stretch*0,\ystretch*-15) {};		

		\node [label={$\cdots$}] (c5) at (\stretch*1,\ystretch*-15) {};
		
		\node [label=above:{$y_{n-1,0}$}] (v17) at (\stretch*2,\ystretch*-12) {};
		\node [label=below:{$y_{n-1,1}$}] (v18) at (\stretch*2,\ystretch*-15) {};
		\node [label=above:{$y_{n,0}$}] (v19) at (\stretch*3,\ystretch*-12) {};
		\node [label=below:{$y_{n,1}$}] (v20) at (\stretch*3,\ystretch*-15) {};

		\foreach \i in {1,2,3,4,5,6,7,8, 9, 10,11, 12,13,14,15,16,17,18,19,20}{
			\fill (v\i) circle (2pt);
		}
		\fill (c1) circle (2pt);
		\fill (c2) circle (2pt);
		
		\path[<->]
				(v1) edge	[swap,red] node  {} (v3)
		(v1) edge	[red] node  [yshift=1.5pt]  {} (v4)
		(v2) edge	[red] node  [yshift=-1.5pt] {} (v3)
		(v2) edge	[swap,red] node  {} (v4)

				(v3) edge	[swap,red] node  {} (v5)
		(v3) edge	[red] node  [yshift=1.5pt]  {} (v6)
		(v4) edge	[red] node  [yshift=-1.5pt] {} (v5)
		(v4) edge	[swap,red] node  {} (v6)
		
		(v7) edge	[swap,red] node  {} (v9)
		(v7) edge	[red] node  [yshift=1.5pt]  {} (v10)
		(v8) edge	[red] node  [yshift=-1.5pt] {} (v9)
		(v8) edge	[swap,red] node  {} (v10)
				
		(v11) edge	[swap,red] node  {} (v13)
		(v11) edge	[red] node  [yshift=1.5pt]  {} (v14)
		(v12) edge	[red] node  [yshift=-1.5pt] {} (v13)
		(v12) edge	[swap,red] node  {} (v14)

				(v13) edge	[swap,red] node  {} (v15)
		(v13) edge	[red] node  [yshift=1.5pt]  {} (v16)
		(v14) edge	[red] node  [yshift=-1.5pt] {} (v15)
		(v14) edge	[swap,red] node  {} (v16)
		
		(v17) edge	[swap,red] node  {} (v19)
		(v17) edge	[red] node  [yshift=1.5pt]  {} (v20)
		(v18) edge	[red] node  [yshift=-1.5pt] {} (v19)
		(v18) edge	[swap,red] node  {} (v20);

		\path[->]

		(v1) edge	[bend right=60,cyan] node  {} (v11)		
		(v11) edge	[bend left=150,cyan] node  {} (v2) 
		(v2) edge	[bend right=60,cyan] node  {} (v12)	
		(v3) edge	[bend right=60,cyan] node  {} (v13)		
		(v13) edge	[bend left=150,cyan] node  {} (v4) 
		(v4) edge	[bend right=60,cyan] node  {} (v14)
		(v5) edge	[bend right=60,cyan] node  {} (v15)		
		(v15) edge	[bend left=150,cyan] node  {} (v6) 
		(v6) edge	[bend right=60,cyan] node  {} (v16)
		(v7) edge	[bend left=60,cyan] node  {} (v17)		
		(v17) edge	[bend right=150,cyan] node  {} (v8) 
		(v8) edge	[bend left=60,cyan] node  {} (v18)
		(v9) edge	[bend left=60,cyan] node  {} (v19)		
		(v19) edge	[bend right=150,cyan] node  {} (v10) 
		(v10) edge	[bend left=60,cyan] node  {} (v20)
		
		(c1) edge	[cyan] node  {} (v1)
		(c1) edge	[cyan] node  {} (v3)
		(c1) edge	[cyan] node  {} (v5)
		(c1) edge	[cyan] node  {} (v7)
		(c1) edge	[cyan] node  {} (v9)		
		
		(v12) edge	[cyan] node  {} (c2)
		(v14) edge	[cyan] node  {} (c2)
		(v16) edge	[cyan] node  {} (c2)
		(v18) edge	[cyan] node  {} (c2)
		(v20) edge	[cyan] node  {} (c2);
		
		\end{tikzpicture}
		\caption{Graph pattern corresponding to the block $B^{t,t'}$ capturing violation of vertical constraints (all nodes have superscript $(t,t')$), which 
		have been omitted for readability.}
		\label{fig:block}
	\end{figure}

We note from the figure above that the graph is a grid, whose breadth (or the number of columns)  is dependent on $n$, and unbounded. 
However, it has bounded height of 6. It has a  tree decomposition which is a path of $n$ nodes : the $i^{th}$ node in the tree decomposition is the bag of 10 nodes $\{x^{t,t'}, y^{t,t'}, x^{t,t'}_{i,0}, x^{t,t'}_{i+1,0}, x^{t,t'}_{i,1}, x^{t,t'}_{i+1,1}, y^{t,t'}_{i,0}, 
y^{t,t'}_{i+1,0}, y^{t,t'}_{i,1}, y^{t,t'}_{i+1,1}\}$.   
 Hence, the  treewidth of each such block is 9 proving the result. \end{proof}

\begin{remark}
  We observe that the queries $Q_1$ and $Q_2$ in Theorem
  \ref{theo:a-Astar} have bounded \textit{treewidth}. Treewidth is
  a commonly used parameter in parameterized complexity analysis and
  intuitively, captures how close the graph is to a tree. A tree has treewidth
  1, while $K_n$, the complete graph on $n$ vertices has treewidth $n-1$. It
  is known that the containment problem of CQs with bounded treewidth (as is
  the evaluation problem of CQs with bounded treewidth) is
  in \ptime \citep{ChekuriR-tcs00}. In this light, it is surprising how the
  complexity of containment increases to \expspace already for $\crpq(a,A^*)$,
  \emph{even for queries of bounded treewidth}.
\end{remark}

\makeatletter{}\section{Deutsch and Tannen's W-Fragment}\label{sec:wfragment}
The complexity of containment of CRPQs with restricted regular expressions has
also been investigated by \citet{DeutschT-dbpl01}. Their work was motivated by
the types of restrictions imposed on navigational expressions in the query
language XPath. Interestingly, they left some questions open, such as the
complexity of containment for
CRPQs using expressions from their \emph{W-fragment}.\footnote{The nomenclature of this fragment
  is a mystery to us. Even Deutsch and Tannen say: ``The fragments called W and
  Z have technical importance but their definitions did not suggest anything
  better than choosing these arbitrary names.''} The
W-fragment is defined by the following grammar:
\begin{align*}
R \quad \rightarrow \quad& \sigma \;\mid\; \_ \;\mid\; S^* \;\mid\; R\cdot R \;\mid\; (R + R)\\
S \quad \rightarrow  \quad& \sigma \;\mid\; \_ \;\mid\; S \cdot S
\end{align*}
Here, $\sigma \in \Sigma$ and $\_$ is a wildcard, i.e., it matches a single,
arbitrary symbol from the infinite set $\Sigma$. In the RPQs underlying
Table~\ref{tab:percentages}, wildcards occurred in 0\% (40 out of 55M)
property paths in Wikidata queries, but in $\sim$4.30\% of the
property paths in valid and in 15.68\% of the property paths in unique
DBpedia$^\pm$ queries. By $\crpq(W)$, we denote CRPQs where the regular
expressions are from the W-fragment.

\citet{DeutschT-dbpl01} claimed that containment for $\crpq(W)$ is \pspace-hard,
but their proof, given in Appendix~C of their article, has a minor error: it
uses the assumption that $\Sigma$, the set of edge labels, is finite. In fact,
we show that 
containment of $\crpq(W)$ queries is in \pitwo. Furthermore, the right query can
even be relaxed completely. \begin{theorem}\label{theo:w-in-crpq}
  Containment of $\crpq(W)$ in \crpq is in \pitwo.
\end{theorem}
\begin{proof}
	Let $Q_1 \in \crpq(W)$ and $Q_2\in \crpq$. We first show a small model
  property. More precisely, we show that whenever there is a counterexample
  to the containment, then there also exists a canonical model $B$ of $Q_1$
  such that $B \notin Q_2$ and $B$ can be represented by a polynomial size graph
  where each edge is either labeled with a single symbol or by $w^i$, where $w$
  is of size linear in $Q_1$ and $i$ is at most $2^{|Q_2|^3}$.
  
  Assume that $B$ is the smallest graph that is a canonical model of $Q_1$ and has no satisfying homomorphism from $Q_2$. \mbox{W.l.o.g.}, we assume that all occurrences of $\_$
  in $Q_1$ are replaced by the same symbol $\$$ that does not occur in $Q_2$. As
  the W-fragments allows only a fixed string below every star, every path of $B$
  can be written as $w_0^{\ell_0}a_1w_1^{\ell_1}a_2 \cdots a_nw_n^{\ell_n}$,
  where $n<|Q_1|$ and $\ell_i \in \nat$, as all long segments of a path have to
  result from applying the Kleene star to a fixed string.

  It remains to show that for every path, all multiplicities are at most
  $2^{|Q_2|^3}$. We assume towards a contradiction that there exists a path $p$
  in $B$, where for some string $w$, the multiplicity $\ell$ is larger than
  $2^{|Q_2|^3}$. We assume \mbox{w.l.o.g.} that all NFAs in $Q_2$ share the same
  transition function $\delta$ over the same set of states $P$, which can be
  achieved by taking the disjoint union of all sets of states. Let $M$ be the
  adjacency matrix of the transition relation for the string $w$, i.e., $M$ is a
  Boolean $|P| \times |P|$ matrix, that has a 1 on position $(i,i')$, if and only
  if $\delta^*(q_i,w)=q_{i'}$. By the pigeonhole principle, there have to be $j$
  and $k$ such that $0 \leq j < k \leq 2^{|P|^2}$ and $M^j = M^k$. We now shorten
  $p$ by $k-j$ copies of $w$ and call the resulting graph $B'$. It is obvious
  that $Q_1$ can still embed into $B'$. We have to show that $Q_2$ cannot embed into $B'$. Towards a
  contradiction we assume that $h$ is a satisfying homomorphism from $Q_2$ to $B'$. Let $p'$ be a subpath of the path $p$ that spans at least $j$ copies of $w$ such
  that no node of $p'$ occurs in the image of $h$. Such a subpath exists due to
  the length of $p$ and the fact that the sizes of $|P|$ and the image of $h$
  are both bounded by $|Q_2|$. We now insert $k-j$ copies of $w$ into $p'$. By definition
  of $M$ and the fact that $M^j=M^k$, we have that $h$ is also a satisfying homomorphism from $Q_2$ to $B$, the desired contradiction.

  We note that the minimal model property implies that the smallest counter
  examples can be stored using only polynomial space by storing the
  multiplicities of strings in binary. The $\pitwo$-algorithm universally
  guesses such a polynomial size representation of a canonical model $B$ of
  $Q_1$. Then it tests whether there exists an homomorphism from $Q_2$ into $B$
  by guessing an embedding. Testing whether a guessed mapping is indeed a
  satisfying homomorphism can be done in polynomial time using the method of fast squaring
  to compute any necessary $\delta^*(q,w^i)$.
\end{proof}

Next we show that, if we assume a \textsl{finite} set of edge labels $\Gamma$ for
knowledge graphs, the containment problem of $\crpq(W)$ is not just \pspace-hard
(as Deutsch and Tannen showed), but even \expspace-complete. The important
technical difference with Theorem~\ref{theo:w-in-crpq} is that, when the labeling alphabet
$\Gamma$ is finite, it is not always possible to replace occurrences of the
wildcard $\_$ with a fresh symbol that doesn't appear in either query.
Therefore, the counterexamples cannot be stored in a compact way.
Even though this is a different setting than all the other results in the paper,
we provide a proof, because the problem was left open by \citet{DeutschT-dbpl01}.

\begin{proposition}\label{prop:w}
  If edge labels of knowledge bases come from a finite alphabet $\Gamma$, then
  containment of $\crpq(W)$ in $\crpq(W)$ is \expspace-complete.
\end{proposition}
\begin{proof}
	To avoid confusion with an infinite alphabet, we write $\Gamma$ instead of
  $\_$.
    We change the languages used in the proof of
  Theorem~\ref{theo:a-Astar}.
  We apply the following homomorphism $h$ to all single label languages of $Q_1$
  and $Q_2$ (including the languages resulting from the double-self-loops in
  Figure~\ref{fig:lower-A-Astar:q1}): $\# \mapsto \varepsilon$, $\$ \mapsto
  \$\tilen\tiley$, $\sigma \mapsto \sigma \tiley\tiley$ for $\sigma \in \mathbb{B} \setminus
  \{\$,\#\}$, and $\sigma \mapsto \sigma \tilen\tilen \in \mathbb{A} \setminus
  \mathbb{B}$, where $\tiley$ and $\tilen$ are new
  symbols, i.e., we encode every symbol $\sigma$ of our original construction by
  the three symbols $\sigma\sigma_1\sigma_2$, where $\sigma_1,\sigma_2 \in \{\tiley,\tilen\}$
  encode whether $\sigma$ belongs to $\mathbb{B}$ and $\mathbb{B}_{\overline
    \$}$, respectively.

  We replace every occurrence of $\mathbb{B}^*$ with the language $(\Gamma \Gamma \tiley)^*$
  and every occurrence of $\mathbb{B}^*_{\overline \$}$ with the language $(\Gamma \tiley
  \tiley)^*$. We replace $e_i^a$ as used in Figure~\ref{fig:lower-A-Astar:gi} with
  $({(0+1)}\tiley\tiley)^{i-1}a\tiley\tiley((0+1)\tiley \tiley)^{n-i-1}$.

  The last change is that we add further bad patterns to the construction of
  $Q_2$ that detects whenever the language $(\Gamma \Gamma \tiley)^*$ resulting from the
  $\mathbb{B}^*$ in $Q_1$ produces an invalid pattern, i.e., a triple that is not
  in the image of $h$. 
\end{proof}

\makeatletter{}\section{Related Work}\label{sec:relatedwork}

The most relevant work to us is that of \citet{CalvaneseGLV-kr00}, who proved
that containment for conjunctive regular path queries, with or without inverses,
is \expspace-complete, generalizing the \expspace upper bound for CRPQs of
\citet{FlorescuLS-pods98}.

\citet{DeutschT-dbpl01} have also studied the containment problem for CRPQ with
restricted classes of regular expressions. They chose fragments of regular
expressions based on expressions in query languages for XML, such as StruQL,
XML-QL, and XPath.
 The fragments they propose are orthogonal to the ones we study here. This is because they allow wildcards and union of words as long as they are not under a Kleene star, while we disallow wildcards and allow union of letters under Kleene star. 
Concretely, they allow $(aa + b)$, which we forbid. On the
other hand, their fragments $(*,_,l^*,|)$
and W do not allow unions under
Kleene star, i.e., they cannot express $(a+b)^*$. Their fragments $Z$ and full
CRPQs allow unions under Kleene star, but are already
\expspace-complete. 
\citet{FlorescuLS-pods98} studied a fragment of conjunctive regular path queries
with wildcards for which the containment problem is \np-complete---thus, it has
the same complexity as containment for conjunctive queries. In their fragment,
they only allow single symbols, transitive closure over wildcards, and
concatenations thereof.

\citet{MiklauS-jacm04} were the first to investigate containment and
satisfiability of \emph{tree pattern queries}, which are acyclic versions of the
\crpqs studied by \citet{FlorescuLS-pods98}. Tree pattern queries are primarily
considered on tree-structured data, but the complexity of their containment
remains the same if one allows graph-structured data
\citep{MiklauS-jacm04,CzerwinskiMNP-jacm18}. Containment of tree pattern queries
was considered in various forms in
\citep{MiklauS-jacm04,NevenS-lmcs06,Wood-icdt03,CzerwinskiMPP-pods15}. 

\citet{BjorklundMS-jcss10} studied containment of conjunctive queries over
tree-structured data and and proved a trichotomy, classifying the problems as in
\ptime, \conp-complete, or \pitwo-complete. Their results cannot be lifted to
general graphs since they use that, if a child has two direct ancestors, then
they must be identical.

\citet{SagivY-jacm1980} studied the equivalence and therefore the containment
problem of relational expressions with query optimization in mind. They show
that when select, project, join, and union operators are allowed, containment is
\pitwo-complete.

\citet{ChekuriR-tcs00} showed that containment of conjunctive queries is in
\ptime when the right-hand side has bounded treewidth. More precisely, they give
an algorithm that runs in $(|Q_1|+|Q_2|)^k$, where $k$ is the width of $Q_2$. So
their algorithm especially works for acyclic queries.

\citet{CalvaneseGLV-dblp01} provide a \pspace-algorithm for containment of
tree-shaped CRPQs with inverses. The algorithm also works if only the right-hand
side is tree-shaped. \citet{Figueira-hal19} shows that containment of UC2RPQs is
in \pspace if the class of graphs considered has ``bounded bridgewidth'' (= size
of minimal edge separator is bounded) and is \expspace-complete otherwise.
\citet{BarceloFR-icalp19} studied the boundedness problem of UC2RPQs and prove
that its \expspace-completeness already holds for CRPQs. (A UC2RPQ is bounded if
it is equivalent to a union of conjunctive queries.)

The practical study of \citep{BonifatiMT-www19} that we mentioned in the
beginning of the paper and that was crucial for the motivation of this work
would not have been possible without the efforts of the Dresden group on
Knowledge-Based Systems \citep{MalyshevKGGB18}, who made sure that anonymized query logs from Wikidata
could be released. \citet{BonifatiMT-www19} studied the same log files as \citet{BielefeldtGK18}.

It should be noted that several extensions and variants of CRPQs have been
studied in the literature. Notable examples are nested regular expressions
\citep{PerezAG-jws10}, CRPQs with node- and edge-variables \citep{BarceloLR-jacm14}, regular
queries \citep{ReutterV15}, and GXPath \citep{LibkinMV16}.

\makeatletter{}\section{Conclusions and Further Work}\label{sec:conclusions}
We have provided an overview of the complexity of CRPQ containment in the case
where the regular expressions in queries come from restricted, yet widely used
classes in practice. A first main result is that, in the case that transitive
closures are only allowed over single symbols, the complexity of \crpq
containment drops significantly. Second, we have shown that even when the
regular expressions are from the restricted class $\crpq(a,A^*)$, the
containment problem remains \expspace-hard. However, contrary to the lower bound
reduction of \citet{CalvaneseGLV-kr00}, the shape of queries (i.e., its
underlying graph) is quite involved, and it crucially involves cycles. This
immediately raises a number of questions.
\begin{itemize}
\item What is the complexity of \querycon of $\crpq(a,A^*)$ in $\crpq(a,A^*)$
  if one of the sides is only a path or a DAG? 
	\item If one takes a careful look at our results, we actually settle the
  complexity of all forms of containment $\cF_1 \subseteq \cF_2$ where $\cF_i$
  is one of our considered classes, except the cases of \querycon of
  $\crpq(a,A^*)$ in $\crpq(A)$ and \querycon of $\crpq(a,A^*)$ in
  $\crpq(a,a^*)$. What is the complexity in these cases?
\end{itemize}

Of course, it would be interesting to understand which of our results can be
extended towards C2RPQs, which would slightly increase the coverage of the
queries we consider in Table~\ref{tab:complexity-smaller}. We believe that all
our upper bounds can be extended and we plan to incorporate these results in an
extended version of the paper.

Another direction could be to combine our fragments with arithmetic constraints.
There is a lot of work done considering query containment of conjunctive queries
with arithmetic constraints (which is \pitwo-complete), see for example
\citet{Afrati-ideas19} and the related work mentioned there. We would like to
understand to which extent such constraints can be incororated without
increasing the complexity of containment.

It would also be interesting to investigate the problem of \emph{boundedness} \citep{BarceloFR-icalp19} 
for the
studied classes of $\crpq$; understanding whether a query is `local' might be of
interest for the graph exploration during its evaluation.


 {\small
 \bibliography{references}
 }

\onecolumn
\newpage

\appendix

\end{document}